\documentclass[10pt,usletter]{article}

\usepackage[totalwidth=450pt,totalheight=640pt]{geometry}

\usepackage{amsmath}
\usepackage{amssymb}
\usepackage{amsthm}
\usepackage{pgf}
\usepackage{booktabs}

\renewcommand{\phi}{\varphi}

\newcommand{\naturals}{\mathbb{N}}

\newcommand{\union}{\cup} 		

\newcommand{\setdiff}{-}
\newcommand{\card}[1]{{|#1|}}		
\newcommand{\set}[1]{\{{#1}\}}          
\newcommand{\st}{\ |\ }		     	

\newcommand{\tuple}[1]{\langle{#1}\rangle}  
\newcommand{\seq}[1]{\langle #1 \rangle}

\theoremstyle{definition}

\theoremstyle{plain}
\newtheorem{theorem}{Theorem}
\newtheorem{lemma}{Lemma}
\newtheorem{corollary}{Corollary}
\newtheorem{proposition}{Proposition}
\newtheorem{claim}{Claim}

\newcommand{\instance}[1]{{\mathbb{#1}}}

\newcommand{\cc}[1]{{\mbox{\textnormal{\textbf{#1}}}}}  
\newcommand{\cocc}[1]{{\mbox{\textrm{co}\textbf{#1}}}}  

\newcommand{\poly}{\cc{P}}
\newcommand{\NP}{\cc{NP}}
\newcommand{\coNP}{\cocc{NP}}
\newcommand{\PSPACE}{\cc{PSPACE}}

\newcommand{\FPT}{\cc{FPT}}
\newcommand{\Weft}{{\cc{W}}}
\newcommand{\W}[1]{{\Weft}{{[#1]}}}
\newcommand{\NOPOLYKERNEL}{{\textnormal\coNP \subseteq \NP/\textup{poly}}}

\newcommand{\problem}[1]{\textsc{#1}}
\newcommand{\SAT}{\problem{SAT}}

\newcommand{\insti}{\instance{I}}
\newcommand{\iplan}{\instance{P}}      

\newcommand{\unpar}[1]{[{#1}]}  

\newcommand{\strips}{{\textsc{Strips}}}
\newcommand{\sasplus}{{SAS$^+$}}

\newcommand{\vars}{V} 
\newcommand{\dom}{D}  
\newcommand{\acts}{A}  
\newcommand{\init}{I}  
\newcommand{\goal}{G}  
\newcommand{\pre}{\mathrm{pre}}  
\newcommand{\eff}{\mathrm{eff}}  %

\newcommand{\varval}{x}

\newcommand{\proj}[2]{{#1[{#2}]}}

\newcommand{\undef}{\mathbf{u}}

\newcommand{\plan}{\omega}

\newcommand{\PE}{\textsc{\sasplus\ Planning}}
\newcommand{\BPE}{\textsc{Bounded \sasplus\ Planning}}
\newcommand{\BPER}[1]{$\set{#1}$-\BPE}

\newcommand{\FOMC}{\textsc{FO\linebreak[3]\ Model\linebreak[3]\ Checking}}
\newcommand{\FOMCR}[1]{$#1$-\FOMC}

\newcommand{\occs}{\mathbf{O}}
\newcommand{\orders}{\mathbf{P}}
\newcommand{\links}{\mathbf{L}}
\newcommand{\link}[3]{{#1\stackrel{#2}{\longrightarrow}{#3}}}

\newcommand{\pbox}[1]{\makebox[1em][l]{#1}}

\newcommand{\hy}{\hbox{-}\nobreak\hskip0pt}

\newcommand{\SB}{\{\,}%
\newcommand{\SM}{\;{|}\;}%
\newcommand{\SE}{\,\}}%
\newcommand{\Card}[1]{\card{#1}}

\newcommand{\mdef}{\;=\;}

\newcommand{\myindent}{~\hspace{15mm}}

\newcommand{\AAA}{\mathcal{A}}

\newcommand{\VVV}{\vars}
\newcommand{\OOO}{\acts}

\newcommand{\SAS}{\mbox{\sasplus}}

\newcommand{\concat}{,}

\newcommand{\undv}{\undef}

\newcommand{\varR}{\textup{VAR}}
\newcommand{\domR}{\textup{DOM}}

\newcommand{\initR}{\textup{INIT}}
\newcommand{\goalR}{\textup{GOAL}}
\newcommand{\goalRV}{\textup{GOAL\_V}}

\newcommand{\preR}{\textup{PRE}}
\newcommand{\postR}{\textup{EFF}}
\newcommand{\preRV}{\textup{PRE\_V}}
\newcommand{\postRV}{\textup{EFF\_V}}

\newcommand{\diff}{\textup{diff}}
\newcommand{\diffR}{\textup{DIFF\_ACT}}
\newcommand{\diffGoalR}{\textup{DIFF\_GOAL}}
\newcommand{\dumR}{\textup{DUM}}
\newcommand{\dumop}{\textup{dum\_a}}
\newcommand{\dumopR}{\textup{DUM\_A}}

\newcommand{\opeNS}{\textup{ACT}}
\newcommand{\varRNS}{\textup{VAR}}
\newcommand{\domRNS}{\textup{DOM}}

\newcommand{\initRVNS}{\textup{INIT\_V}}

\newcommand{\goalRVNS}{\textup{GOAL\_V}}

\newcommand{\preRNS}{\textup{PRE}}
\newcommand{\postRNS}{\textup{EFF}}
\newcommand{\preRVNS}{\textup{PRE\_V}}
\newcommand{\postRVNS}{\textup{EFF\_V}}

\newcommand{\diffRNS}{\textup{DIFF\_ACT}}
\newcommand{\diffGoalRNS}{\textup{DIFF\_GOAL}}
\newcommand{\dumRNS}{\textup{DUM}}
\newcommand{\dumopRNS}{\textup{DUM\_A}}

\newcommand{\fvalue}{\textit{value}}
\newcommand{\fcheckpre}{\textit{check-pre}}
\newcommand{\fcheckpreall}{\textit{check-pre-all}}
\newcommand{\fcheckpregoal}{\textit{check-goal}}

\newcommand{\fcheckpost}{\textit{check-eff}}
\newcommand{\fdiffop}{\textit{diff-op}}
\newcommand{\fdiffopall}{\textit{diff-op-all}}
\newcommand{\fdiffgoal}{\textit{diff-goal}}

\title{A Complete Parameterized Complexity Analysis of \\Bounded Planning}

\author{
Christer B\"{a}ckstr\"{o}m$^1$,
Peter Jonsson$^1$,
Sebastian Ordyniak$^2$, and
Stefan Szeider$^3$\\[0.1cm]
\mbox{}\small$^1$Department of Computer Science, Link{\"o}ping University,
Link{\"o}ping, Sweden\\[-4pt]
\small christer.backstrom@liu.se, peter.jonsson@liu.se\\
\mbox{}\small$^2$ Faculty of Informatics, Masaryk University, Brno, Czech Republic\\[-4pt]
\small sordyniak@gmail.com\\
\mbox{}\small$^3$Institute of Information Systems, Vienna University of Technology,
Vienna, Austria\\[-4pt]
\small stefan@szeider.net
}

\date{}

\begin{document}

\maketitle

\begin{abstract}
  The {\em propositional planning} problem is a notoriously difficult
  computational problem, which remains hard even under strong
  syntactical and structural restrictions.  Given its difficulty it
  becomes natural to study planning in the context of parameterized
  complexity. In this paper we continue the work initiated by Downey,
  Fellows and Stege on the parameterized complexity of planning with
  respect to the parameter ``length of the solution plan.'' We provide
  a complete classification of the parameterized complexity of the
  planning problem under two of the most prominent syntactical
  restrictions, i.e., the so called PUBS restrictions introduced by
  B{\"a}ckstr\"{o}m and Nebel and restrictions on the number of
  preconditions and effects as introduced by Bylander. We also
  determine which of the considered fixed-parameter tractable problems
  admit a polynomial kernel and which don't.
\end{abstract}


\thispagestyle{empty}
\section{Introduction}
\noindent
The (propositional) planning problem has been the subject of intensive
study in knowledge representation, artificial intelligence and control
theory and is relevant for a large number of industrial
applications~\cite{GhallabNauTraverso04}.  The problem involves deciding
whether an \emph{initial state}---an $n$-vector over some domain 
$D$–--can
be transformed into a \emph{goal state} via the application of {\em
  actions} (or {\em operators}) each consisting of 
 {\em preconditions} and {\em
  post-conditions} (or {\em effects}) stating the conditions that need
to hold before the
action can be applied and which conditions will hold after the
application of the action, respectively.
It is known that the problem of deciding whether a solution exists or not is
\PSPACE-complete~\cite{BackstromNebel95,Bylander94}.
Although various \NP-complete and even tractable restrictions are
known in the literature
\cite{BackstromNebel95,BrafmanDomshlak03,Bylander94,JonssonBackstrom98} 
these are often considered not to coincide
well with cases that are intersting in practice.
In the authors experience, classical complexity analysis is often
too coarse to give relevant results for planning, since most
interesting restrictions seem to remain \PSPACE-complete.
Despite this, there has been very few attempts 
to use alternative analysis methods.
The few exceptions include 
probabilistic analysis~\cite{Bylander96},
approximation~\cite{BetzHelmert09,Jonsson99} and
padding~\cite{BackstromJonsson11}. 

Another obvious alternative is to use the framework of Parameterized
Complexity which offers the more relaxed notion of
\emph{fixed-parameter tractability} (FPT).  A problem is
fixed-parameter tractable if it can be solved in time $f(k)n^{O(1)}$
where $f$ is an arbitrary function of the parameter $k$ (which
measures some aspect of the input) and $n$ is the input size. Indeed,
already in a 1999 paper, Downey, Fellows and Stege
\cite{DowneyFellowsStege99} initiated the parameterized analysis of
planning, taking the minimum number of steps from the initial state to
the goal state (i.e., the length of the solution plan) as the
parameter. However, the parameterized viewpoint did not immediately
gain momentum for analysis of planning and it is only during the last
few years that we have witnessed a strongly increased interest in this
method, cf.~\cite{BackstromJonsson11,BackstromJonssonStahlberg13,
  deHahnRoubickovaSzeider13,KroneggerPfandlerPichler13}.

In this article, we use the same parameter as Downey et al.
and provide a complete analysis of planning under various
syntactical restrictions, in particular the restrictions
considered by Bylander~\cite{Bylander94} and by B\"{a}ckstr\"{o}m and
Nebel~\cite{BackstromNebel95}.
These were among the first attempts to
understand why and when planning is hard or easy and they have
had a heavy influence on later theoretical research in planning.
We complement these results by also considering bounds on 
problem kernels for those planning problems that we prove to
be fixed-parameter tractable.
It is known that a decidable problem is fixed-parameter
tractable if and only if it admits a polynomial-time self-reduction 
where the size
of the resulting instance is bounded by a function~$f$ 
of the parameter~\cite{Fellows06,Fomin10,GuoNiedermeier07}.
The function $f$ is called the {\em kernel size}. By providing upper and 
lower bounds
on the kernel size, one can rigorously establish the potential of 
polynomial-time
preprocessing for the problem at hand.

\subsection*{Our results}

We provide a full parameterized complexity analysis of 
planning with respect to the length of the solution plan, under 
all combinations of the syntactical P, U, B, and S restrictions 
previously considered by B\"{a}ckstr\"{o}m and 
Nebel~\cite{BackstromNebel95}
as well as under the restrictions on the number of preconditions 
and effects previously considered by Bylander~\cite{Bylander94}.
Our new parameterized results are summarized 
in Table~\ref{table:bylander} and Figure~\ref{fig:pubs-lattice}
alongside the previously reported classical complexity results.
We discuss our results more thoroughly in Section~\ref{sec:class}.
In addition, we examine whether the fixed-parameter tractable
subcases, which we obtain, admit polynomial kernels or not. 
Our results on this are negative
throughout---if any of these problems admit polynomial kernels,
or even polynomial bi-kernels,
then $\NOPOLYKERNEL$ and the Polynomial-time Hierarchy collapses.

\subsection*{Outline}

The rest of the paper is laid out as follows.  Section 2 defines some
concepts of parameterized complexity theory and Section 3 defines the
\sasplus\ and \strips\ planning languages.  
The hardness results are collected
in Section 4 and the membership results in Section 5.
Section 6 is devoted to our tractability results and in Section 7
we show that none of the tractable subcases admits a polynomial
kernel.
We summarize
the results of the paper in Section 8 and discuss some observations and
consequences.  The paper ends with an outlook in Section~9.

\begin{table}[tb]
  \centering
  \begin{tabular}{@{}l@{~~~~~}|llll@{}} 
    \toprule
    & $e=1$     & $e=2$ & fixed $e > 2$ ~~~ & arbitrary $e$\\
    \midrule
    $p=0$       & in \poly   & in \FPT      & \W{1}-C  & \W{2}-C \\ 
    & in \poly    & \NP-C      &\NP-C             & \NP-C \\ 
    \midrule
    $p=1$       & \W{1}-C    & \W{1}-C & \W{1}-C   & \W{2}-C \\
    & \NP-H       & \NP-H      &  \NP-H    & \PSPACE-C \\ 
    \midrule
    fixed $p > 1$ & \W{1}-C    & \W{1}-C & \W{1}-C   & \W{2}-C \\
    & \NP-H       & \PSPACE-C  & \PSPACE-C & \PSPACE-C \\ 
    \midrule
    arbitrary $p$    & \W{1}-C    & \W{1}-C & \W{1}-C   & \W{2}-C \\ 
    & \PSPACE-C   & \PSPACE-C  & \PSPACE-C   & \PSPACE-C \\ 
    \bottomrule
  \end{tabular}
  \medskip

  \caption{Complexity of \BPE{} when
    restricting the number of preconditions ($p$) and 
    effects ($e$).
    All parameterized results are shown in this paper and all
    classical results are
    from Bylander~\cite{Bylander94}. 
    The classical results apply to \strips, while the parameterized
    results hold for \sasplus\ (the hardness results hold already
    for binary domains and the membership results hold for
    arbitrary domains).
    We also show that none of the problems
    that are in \FPT\ admit polynomial kernels.
}

    \label{table:bylander}
\end{table}

\begin{figure}[thb]
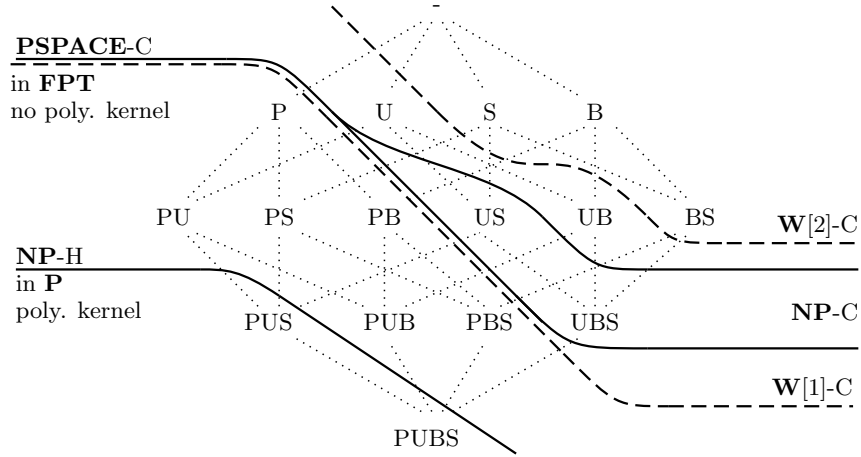

\centering
\begin{pgfpicture}{0cm}{0cm}{12cm}{8cm}


  \pgftranslateto{\pgfxy(6,3.5)}

    \newdimen\pubsdim
    \pgfextractx{\pubsdim}{\pgfpoint{7mm}{0mm}}
    \pgfsetxvec{\pgfpoint{7mm}{0mm}}
    \pgfsetyvec{\pgfpoint{0mm}{7mm}}

    \small
    \pgfsetlinewidth{0.2mm}
    \pgfsetdash{{0.2mm}{0.8mm}}{0mm}

    \pgfputat{\pgfxy(0,4)}{\pgfbox[center,center]{-}}

    \pgfline{\pgfxy(-0.3,3.7)}{\pgfxy(-2.7,2.3)} 
    \pgfline{\pgfxy(-0.1,3.7)}{\pgfxy(-0.9,2.4)} 
    \pgfline{\pgfxy(0.1,3.7)}{\pgfxy(0.9,2.4)}   
    \pgfline{\pgfxy(0.3,3.7)}{\pgfxy(2.7,2.3)}   

    \pgfputat{\pgfxy(-3,2)}{\pgfbox[center,center]{P}}
    \pgfputat{\pgfxy(-1,2)}{\pgfbox[center,center]{U}}
    \pgfputat{\pgfxy(1,2)}{\pgfbox[center,center]{S}}
    \pgfputat{\pgfxy(3,2)}{\pgfbox[center,center]{B}}

    \pgfline{\pgfxy(-3.4,1.7)}{\pgfxy(-4.8,0.3)} 
    \pgfline{\pgfxy(-3.0,1.7)}{\pgfxy(-3.0,0.3)} 
    \pgfline{\pgfxy(-2.6,1.7)}{\pgfxy(-1.2,0.3)} 

    \pgfline{\pgfxy(-1.4,1.7)}{\pgfxy(-4.6,0.3)} 
    \pgfline{\pgfxy(-0.9,1.7)}{\pgfxy(0.6,0.3)}  
    \pgfline{\pgfxy(-0.6,1.7)}{\pgfxy(2.6,0.3)}  

    \pgfline{\pgfxy(0.8,1.7)}{\pgfxy(-2.6,0.3)}  
    \pgfline{\pgfxy(1.0,1.7)}{\pgfxy(1.0,0.3)}   
    \pgfline{\pgfxy(1.2,1.7)}{\pgfxy(4.6,0.3)}   

    \pgfline{\pgfxy(2.8,1.7)}{\pgfxy(-0.6,0.3)}  
    \pgfline{\pgfxy(3.0,1.7)}{\pgfxy(3.0,0.3)}   
    \pgfline{\pgfxy(3.4,1.7)}{\pgfxy(4.8,0.3)}   

    \pgfputat{\pgfxy(-5,0)}{\pgfbox[center,center]{PU}}
    \pgfputat{\pgfxy(-3,0)}{\pgfbox[center,center]{PS}}
    \pgfputat{\pgfxy(-1,0)}{\pgfbox[center,center]{PB}}
    \pgfputat{\pgfxy(1,0)}{\pgfbox[center,center]{US}}
    \pgfputat{\pgfxy(3,0)}{\pgfbox[center,center]{UB}}
    \pgfputat{\pgfxy(5,0)}{\pgfbox[center,center]{BS}}
 
    \pgfline{\pgfxy(-3.4,-1.7)}{\pgfxy(-4.8,-0.3)} 
    \pgfline{\pgfxy(-3.0,-1.7)}{\pgfxy(-3.0,-0.3)} 
    \pgfline{\pgfxy(-2.6,-1.7)}{\pgfxy(0.8,-0.3)} 

    \pgfline{\pgfxy(-1.4,-1.7)}{\pgfxy(-4.6,-0.3)} 
    \pgfline{\pgfxy(-1.0,-1.7)}{\pgfxy(-1.0,-0.3)}  
    \pgfline{\pgfxy(-0.6,-1.7)}{\pgfxy(2.6,-0.3)}  

    \pgfline{\pgfxy(0.6,-1.7)}{\pgfxy(-2.6,-0.3)}  
    \pgfline{\pgfxy(0.9,-1.7)}{\pgfxy(-0.8,-0.3)}  
    \pgfline{\pgfxy(1.2,-1.7)}{\pgfxy(4.6,-0.3)}   

    \pgfline{\pgfxy(2.8,-1.7)}{\pgfxy(0.8,-0.3)}  
    \pgfline{\pgfxy(3.0,-1.7)}{\pgfxy(3.0,-0.3)}  
    \pgfline{\pgfxy(3.2,-1.7)}{\pgfxy(4.8,-0.3)}  

    \pgfputat{\pgfxy(-3.2,-2)}{\pgfbox[center,center]{PUS}}
    \pgfputat{\pgfxy(-0.9,-2)}{\pgfbox[center,center]{PUB}}
    \pgfputat{\pgfxy(1,-2)}{\pgfbox[center,center]{PBS}}
    \pgfputat{\pgfxy(3,-2)}{\pgfbox[center,center]{UBS}}

    \pgfline{\pgfxy(-0.3,-3.7)}{\pgfxy(-2.7,-2.3)} 
    \pgfline{\pgfxy(-0.1,-3.7)}{\pgfxy(-0.9,-2.4)} 
    \pgfline{\pgfxy(0.1,-3.7)}{\pgfxy(0.9,-2.4)}   
    \pgfline{\pgfxy(0.3,-3.7)}{\pgfxy(2.7,-2.3)}   

    \pgfputat{\pgfxy(-0.2,-4.2)}{\pgfbox[center,center]{PUBS}}


    \pgfsetlinewidth{0.3mm}
    \pgfsetdash{}{0mm}

    \pgfxyline(-8.0,-1.0)(-4.5,-1.0)
    \pgfxycurve(-4.5,-1.0)(-4.0,-1.0)(-3.75,-1.0)(-3.0,-1.5)
    \pgfxyline(-3.0,-1.5)(1.5,-4.5)
    \pgfstroke
    \pgfputat{\pgfxy(-8.0,-1.1)}{\pgfbox[left,top]{in \poly}}
    \pgfputat{\pgfxy(-8.0,-1.6)}{\pgfbox[left,top]{poly. kernel}}

    \pgfxyline(-8.0,3.0)(-4.0,3.0)
    \pgfxycurve(-4.0,3.0)(-3.0,3.0)(-3.0,3.0)(-2.0,2.0)
    \pgfxyline(-2.0,2.0)(2.0,-2.0)
    \pgfxycurve(1.5,-1.5)(2.5,-2.5)(2.5,-2.5)(4.0,-2.5)
    \pgfxyline(4.0,-2.5)(8.0,-2.5)
    \pgfputat{\pgfxy(-8.0,-0.9)}{\pgfbox[left,bottom]{\NP-H}}
    \pgfputat{\pgfxy(8.0,-1.8)}{\pgfbox[right,center]{\NP-C}}
    \pgfxycurve(-2.0,2.0)(-1.0,1.0)(1.0,1.0)(2.0,0.0)
    \pgfxycurve(2.0,0.0)(3.0,-1.0)(3.0,-1.0)(4.0,-1.0)
    \pgfxyline(4.0,-1.0)(8.0,-1.0)
    \pgfstroke
    \pgfputat{\pgfxy(-8.0,3.1)}{\pgfbox[left,bottom]{\PSPACE-C}}
    \pgfsetlinewidth{0.3mm}
    \pgfsetdash{{2mm}{1mm}}{0mm}

    \begin{pgftranslate}{\pgfxy(-0.1,-0.1)}
      \pgfxyline(-8.0,3.0)(-4.0,3.0)
      \pgfxycurve(-4.0,3.0)(-3.0,3.0)(-3.0,3.0)(-2.0,2.0)
      \pgfxyline(-2.0,2.0)(2.5,-2.5)
      \pgfxycurve(2.5,-2.5)(3.5,-3.5)(3.5,-3.5)(4.5,-3.5)
      \pgfxyline(4.5,-3.5)(8.0,-3.5)
      \pgfstroke
      \pgfputat{\pgfxy(-8.0,2.8)}{\pgfbox[left,top]{in \FPT}}
      \pgfputat{\pgfxy(-8.0,2.3)}{\pgfbox[left,top]{no poly. kernel}}
      \pgfputat{\pgfxy(8.0,-3.4)}{\pgfbox[right,bottom]{\W{1}-C}}
    \end{pgftranslate}

    \begin{pgftranslate}{\pgfxy(0.0,0.0)}
      \pgfxyline(-2.0,4.0)(0.0,2.0)
      \pgfxycurve(0.0,2.0)(2.0,0.0)(2.0,2.0)(4.0,0.0)
      \pgfxycurve(4.0,0.0)(4.5,-0.5)(4.5,-0.5)(5.5,-0.5)
      \pgfxyline(5.5,-0.5)(8.0,-0.5)
      \pgfstroke
      \pgfputat{\pgfxy(8.0,-0.4)}{\pgfbox[right,bottom]{\W{2}-C}}
    \end{pgftranslate}
  \end{pgfpicture}
  \caption{Complexity of \BPE\ for the restrictions P, U, B and S 
    illustrated as a lattice defined by all
    possible combinations of these restrictions. Again all
    parameterized results are shown in this paper and all classical results 
    are from B{\"a}ckstr{\"o}m and Nebel~\cite{BackstromNebel95}.
    Furthermore, as shown in this paper, PUS and PUBS are
    the only restrictions that admit a polynomial kernel, unless the
    Polynomial Hierarchy collapses. }
  \label{fig:pubs-lattice}
\end{figure}

\section{Parameterized Complexity}
\noindent
We define the basic notions of Parameterized Complexity and refer to
other sources~\cite{DowneyFellows99,FlumGrohe06} for an in-depth
treatment.  Let $\Sigma$ be a finite alphabet and $\naturals$ the set
of natural numbers.  A \emph{parameterized decision problem} or
\emph{parameterized problem}, for short, is a language $L \subseteq
\Sigma^*\times\naturals$.  The \emph{instances} of the problem are
pairs on the form $\tuple{\insti,k}$, where $\insti$ is a string over
$\Sigma^*$, which constitutes the main part, and $k$ is the
\emph{parameter}.
A parameterized problem is {\em fixed-parameter tractable (FPT)} 
if there exists an algorithm that solves any instance
$\tuple{\insti,k}$ in time $f(k)n^{c}$ where $f$ is an
arbitrary computable function, $n = |\tuple{\insti,k}|$,  and 
$c$ is a constant independent of both $n$ and $k$.  
\FPT\ is the class of all fixed-parameter tractable
parameterized problems. 
Since the emphasis in parameterized algorithms
lies on the dependence of the running time on $k$ we will sometimes
use the notation $O^*(f(k))$ as a synonym for $O(f(k)n^c)$.

Parameterized complexity offers a completeness theory, similar to the
theory of $\NP$-completeness, that supports the accumulation of strong
theoretical evidence that certain parameterized problems are not
fixed-parameter tractable. This theory is based on a hierarchy of
complexity classes
\[\FPT \subseteq \W{1} \subseteq \W{2} \subseteq \W{3} \subseteq
\cdots\] where all inclusions are believed to be strict. Each class
$\W{i}$ contains all parameterized problems that can be reduced by a
parameterized reduction to a certain canonical parameterized problem
(known as {\sc Weighted $i$-Normalized Satisfiability}) under {\em
  parameterized} reductions. 

A \emph{parameterized reduction} or \emph{fpt-reduction} from  
a parameterized problem $P$ to a
parameterized problem $Q$ is an algorithm that maps instances 
$\tuple{\insti,k}$ of $P$ to instances $\tuple{\insti',k'}$ of $Q$
such that:
\begin{enumerate}
\item $\tuple{\insti,k} \in P$ if and only if $\tuple{\insti',k'} \in
  Q$;
\item there is a computable function $g$ such that $k' \leq g(k)$; and
\item there is a computable function $f$ and a constant $c$ such that
  $\tuple{\insti,k}$ is computed in time $O(f(k) \cdot n^c)$, where $n
  = |\tuple{\insti,k}|$.
\end{enumerate}

A \emph{bi-kernelization} \cite{AlonGutinKimSzeiderYeo11} (or
\emph{generalized kernelization} \cite{BodlaenderDowneyFellowsHermelin09})
for a parameterized problem~$P$ is a parameterized reduction from $P$ to a
parameterized problem $Q$ 
that maps instances 
$\tuple{\insti,k}$ of $P$ to instances $\tuple{\insti',k'}$ of $Q$
with the additional property that
\begin{enumerate}
\item $\tuple{\insti',k'}$ can be computed in time that is polynomial in
  $\Card{\insti}+k$, and 
\item   $|\insti'|$ and $k'$ are both bounded by some function $f$~of~$k$.  
\end{enumerate}
The output $\tuple{\insti',k'}$ is called a \emph{bi-kernel}
(or a \emph{generalized kernel}). We say that $P$
has a \emph{polynomial bi-kernel} if $f$ is a polynomial.  
If $P=Q$, we call the bi-kernel a \emph{kernel}.
Every fixed-parameter tractable problem admits a bi-kernel, 
but not necessarily a polynomial bi-kernel~\cite{CaiEtAl93b}.


A \emph{polynomial parameter reduction} from a parameterized problem $P$
to a parameterized problem $Q$ 
is a parameterized reduction from $P$ to $Q$ 
that maps instances $\tuple{\insti,k}$ of
$P$ to instances $\tuple{\insti',k'}$ of $Q$ with the additional
property that 
\begin{enumerate}
\item $\tuple{\insti',k'}$ can be computed in time that is polynomial in
  $\Card{\insti}+k$, and 
\item  $k'$ is bounded by some polynomial $p$~of~$k$.  
\end{enumerate}

\noindent The following result is an adaptation of a result by
Bodlaender~\cite[Theorem~8]{Bodlaender09}.
\begin{proposition}\label{pro:poly-par-reduction-bi}
  Let $P$ and $Q$ be two parameterized problems such that there is a
  polynomial parameter reduction from $P$ to $Q$. Then, if $Q$
  has a polynomial bi-kernel also $P$ has a polynomial bi-kernel.
\end{proposition}
\begin{proof}
  Let $(\insti,k)$ be an instance of $P$. We first apply the
  polynomial parameter reduction from $P$ to $Q$ to the instance
  $(\insti,k)$ and obtain the instance $(\insti',k')$ of~$Q$. Then
  the instance $(\insti'',k'')$ of some parameterized problem, say $Q'$,
  that we obtain by applying the polynomial bi-kernelization algorithm for $Q$ is also
  a polynomial bi-kernel for $P$. This concludes the proof of the
  proposition.
\end{proof}

For a parameterized problem $P \subseteq \Sigma^*\times\naturals$, 
we define its 
\emph{unparameterized version} $\unpar{P}$ as the corresponding
classical problem with the parameter given in unary.
That is, every instance  $\tuple{\insti,k}$ of $P$ has
a corresponding instance $\unpar{\tuple{\insti,k}}$ of $\unpar{P}$
such that $\unpar{\tuple{\insti,k}}$ is the string $\insti\#1^k$,
where $\#\not\in\Sigma$ is a separator symbol and $1$ is 
an arbitrary symbol in $\Sigma$.


We note here that most researchers in
kernelization talk about kernels instead of bi-kernels. It is however
well known that the approaches to obtain lower bounds for kernels and
bi-kernels, respectively, work in the same manner~\cite{BodlaenderDowneyFellowsHermelin09}.
It is immediate from the definitions that if a problem does not admit
a polynomial bi-kernel, then it cannot admit a polynomial kernel
either, so super-polynomial lower bounds for the size of bi-kernels
imply super-polynomial lower bounds for the size of kernels.  Since we will be
mostly concerned with lower bounds we will give all our results in
terms of bi-kernels.

An \emph{OR-composition algorithm} for a parameterized problem $P$ maps
$t$ instances $\tuple{\insti_1,k},\dotsc,\tuple{\insti_t,k}$ of $P$ to
one instance $\tuple{\insti',k'}$ of $P$ such that the algorithm runs in
time polynomial in $\sum_{1 \leq i \leq t}|\insti_i|+k$, the parameter
$k'$ is bounded by a polynomial in the parameter $k$, and
$\tuple{\insti',k'} \in P$ if and only if 
there is an $i$, where $1 \leq i \leq t$, such that 
$\tuple{\insti_i,k} \in P$.

\begin{proposition}[Bodlaender, et al.~{\cite[Lemmas~1 and~2]{BodlaenderDowneyFellowsHermelin09}}]\label{pro:or-comp-no-poly-kernel}
  \sloppypar
  If a parameterized problem $P$ has an OR\hy composition algorithm
  and its unparameterized version $\unpar{P}$ is \NP-complete,
  then it has no polynomial bi-kernel\,\footnote{
    The original results refer to kernels but the authors remark that
    all their results extend also to bi-kernels.}
  unless $\NOPOLYKERNEL$.
\end{proposition}


Unfortunately, this proposition can only be applied to problems
that are contained in \NP{}, while membership in \NP{} is not
known for some of the problems that we consider in this
article. Furthermore, to the best of our knowledge, there is no
consistent presentation of bi-kernel lower bounds for problems that
may not be in \NP{}. Hence, we fill this gap by giving such a
presentation from first principles, i.e., in the following we will
give an analogue of  Proposition~\ref{pro:or-comp-no-poly-kernel} 
for problems that are not (or not known to be) in \NP{}. 

A \emph{distillation algorithm}~\cite{BodlaenderDowneyFellowsHermelin09,FortnowSanthanam08} 
for a classical problem $P$ is an
algorithm that takes $t$ instances $\insti_1,\dotsc,\insti_t$ of $P$ as the
input, runs in time polynomial in $\sum_{1 \leq i \leq t}|\insti_i|$, and outputs an
instance $\insti$ of some problem $Q$ such that: 
(1) $\insti \in Q$ if and only if 
there is an~$i$, where $1 \leq i \leq t$, such that 
$\insti_i \in P$ and 
(2)~$|\insti|$ is polynomial in $\max_{1  \leq i \leq t}|\insti_i|$.
\begin{proposition}[{Fortnow and Santhanam~\cite[Theorem 1.2]{FortnowSanthanam08}}]\label{pro:sat-no-distillation}
  \sloppypar
  Unless $\NOPOLYKERNEL$, the satisfiability problem for propositional
  formulas ($\SAT$) has no distillation algorithm.
\end{proposition}
Using the fact that $\SAT$ is \NP\hy complete and there are polynomial
reductions from $\SAT$ to any other \NP\hy hard problem, we
immediately obtain the following corollary.
\begin{corollary}\label{cor:no-distillation}
  Unless $\NOPOLYKERNEL$, no \NP\hy hard problem has a distillation algorithm.
\end{corollary}

Below, we first introduce a stronger OR\hy composition concept and
then prove generalizations to \NP-hard problems of two results known
from the literature.

A \emph{strong OR\hy composition algorithm} for a parameterized
problem $P$ maps~$t$ instances $\tuple{\insti_1,k_1},\dotsc$, \hskip0pt$\tuple{\insti_t,k_t}$ of $P$ to
one instance $\tuple{\insti,k}$ of $P$ such that the algorithm runs in
time polynomial in $\sum_{1 \leq i \leq t}|\insti_i|+\max_{1 \leq i \leq t}k_i$, 
the parameter
$k$ is bounded by a polynomial in $\max_{1 \leq i \leq t}k_i$, and
$\tuple{\insti,k} \in P$ if and only if there is
an $i$, where $1 \leq i \leq t$, such that $\tuple{\insti_i,k_i} \in P$.

The following result is an adaptation of 
Proposition~\ref{pro:or-comp-no-poly-kernel}, based on the original
proof of this result~\cite{BodlaenderDowneyFellowsHermelin09}.
\begin{proposition}\label{pro:strong-or-comp-no-poly-kernel}
  If a parameterized problem $P$ has a strong OR\hy composition algorithm
  and its unparameterized version $\unpar{P}$ is \NP-hard,
  then it has no polynomial bi-kernel unless $\NOPOLYKERNEL$.
\end{proposition}
\begin{proof}
  We show that if $P$ satisfies the conditions of the proposition 
  and $P$ has a polynomial bi-kernel, then $\unpar{P}$ has a
  distillation algorithm and it thus follows from 
  Corollary~\ref{cor:no-distillation} that $\NOPOLYKERNEL$.

  Let $\tuple{\insti_1,k_1},\dotsc,\tuple{\insti_t,k_t}$ be instances
  of $P$ and $\unpar{\tuple{\insti_1,k_1}},\dotsc,\unpar{\tuple{\insti_t,k_t}}$ be the corresponding
  unparameterized instances of $\unpar{P}$. We give a distillation algorithm
  for $\unpar{P}$ that consists of two steps. 

  In the first step the algorithm runs the strong OR\hy composition
  algorithm for $P$ and obtains the instance $\tuple{\insti,k}$ from
  the instances $\tuple{\insti_1,k_1},\dotsc,\tuple{\insti_t,k_t}$. In
  the second step the algorithm runs the polynomial bi-kernelization
  algorithm on $\tuple{\insti,k}$ and obtains the instance
  $\tuple{\insti',k'}$ of a parameterized problem~$Q$. The algorithm
  then outputs $\unpar{\tuple{\insti',k'}}$. 

  In the following we will show that the algorithm outlined above is
  indeed a distillation algorithm for $\unpar{P}$. 
  It is straightforward to verify that 
  $\unpar{\tuple{\insti',k'}} \in \unpar{Q}$ if and only if there is an $i$,
  where $1 \leq i \leq t$, such that $\unpar{\tuple{\insti_i,k_i}} \in \unpar{P}$. 
  The running time of the first step of
  our algorithm is polynomial in 
  $\sum_{1\leq i \leq t}|\insti_i|+\max_{1\leq i \leq t}k_i$ (because of the properties
  of the strong OR\hy composition algorithm), which in turn is polynomial in
  $\sum_{1\leq i \leq t}|\unpar{\tuple{\insti_i,k_i}}|$. 
  The running time of the second step is polynomial in 
  $|\unpar{\tuple{\insti,k}}|$, which is again polynomial in 
  $\sum_{1\leq i \leq t}|\unpar{\tuple{\insti_i,k_i}}|$. 
  Hence, the running time of the complete algorithm is polynomial in 
  $\sum_{1\leq i \leq t}|\unpar{\tuple{\insti_i,k_i}}|$, as required for 
  a distillation algorithm.

  Furthermore, since $\tuple{\insti',k'}$ is a polynomial bi-kernel
  of $\tuple{\insti,k}$, 
  it follows that  $|\unpar{\tuple{\insti',k'}}|$ is bounded by a
  polynomial function in $k'$, and hence in $k$. Because of the properties
  of the strong
  OR\hy composition algorithm we also obtain that $k$ is bounded by a
  polynomial in $\max_{1 \leq i \leq t}k_i$ which in turn is bounded
  by $\max_{1\leq i \leq t}|\unpar{\tuple{\insti_i,k_i}}|$. Putting all this
  together, we obtain that  $|\unpar{\tuple{\insti',k'}}|$ is bounded by
  a polynomial function of $\max_{1\leq i \leq t}|\unpar{\tuple{\insti_i,k_i}}|$. 
  This shows that our algorithm is a
  distillation algorithm for $\unpar{P}$, which, because of
  Corollary~\ref{cor:no-distillation}, implies that $\NOPOLYKERNEL$. 
\end{proof}



\section{Planning Framework}\label{sec:planning-framework}
\noindent
Let $\vars = \set{v_1,\ldots,v_n}$ be a finite set of \emph{variables}
over a finite \emph{domain} $\dom$.  Implicitly define $\dom^+ = \dom
\union \set{\undef}$, where $\undef$ is a special ``undefined'' value not present in
$\dom$.  Then $\dom^n$ is the set of \emph{total states} and
$(\dom^+)^n$ is the set of \emph{partial states} over $\vars$ and
$\dom$, where $\dom^n \subseteq (\dom^+)^n$.  The value of a variable
$v$ in a state $s \in (\dom^+)^n$ is denoted $\proj{s}{v}$.
A \emph{\sasplus\ instance} is a tuple $\iplan =
\tuple{\vars,\dom,\acts,\init,\goal}$ where $\vars$ is a set of
variables, $\dom$ is a domain, $\acts$ is a set of \emph{actions},
$\init \in \dom^n$ is the \emph{initial state} and $\goal \in
(\dom^+)^n$ is the \emph{goal state}. 
Each
action $a \in \acts$ has a \emph{precondition} $\pre(a) \in
(\dom^+)^n$ and an \emph{effect} $\eff(a) \in (\dom^+)^n$.  We will
frequently use the convention that a variable has value $\undef$ in a
precondition/effect unless a value is explicitly specified.  Let $a
\in \acts$ and let $s \in \dom^n$.  Then $a$ is \emph{valid in $s$} if
for all $v \in \vars$, either $\proj{\pre(a)}{v} = \proj{s}{v}$ or
$\proj{\pre(a)}{v} = \undef$.  Furthermore, the \emph{result of $a$ in
$s$} is a state $t \in \dom^n$ defined such that for all $v \in
\vars$, $\proj{t}{v} = \proj{\eff(a)}{v}$ if $\proj{\eff(a)}{v} \neq
\undef$ and $\proj{t}{v} = \proj{s}{v}$ otherwise.

Let $s_0, s_\ell \in \dom^n$ and let $\plan = \seq{a_1,\ldots,a_\ell}$
be a sequence of actions.  Then $\plan$ is a \emph{plan from $s_0$ to
$s_\ell$} if either
\begin{itemize}
\item[(1)] $\plan = \seq{}$ and $\ell = 0$ or
\item[(2)] there are states $s_1,\ldots,s_{\ell-1} \in \dom^n$ such
  that for all $1 \leq i \leq \ell$, $a_i$ is valid in
  $s_{i-1}$ and $s_i$ is the result of $a_i$ in $s_{i-1}$.  
\end{itemize}
A state $s
\in \dom^n$ is a \emph{goal state} if for all $v \in \vars$, either
$\proj{\goal}{v} = \proj{s}{v}$ or $\proj{\goal}{v} = \undef$.  An
action sequence $\plan$ is a \emph{plan for $\iplan$} if it is a plan
from $\init$ to some goal state.  
We will study the following problem:

\smallskip

\begin{quote}
\noindent \BPE\\ \textit{Instance:} A tuple $\tuple{\iplan,k}$ where
$\iplan$ is a \sasplus\ instance and $k$ is a positive integer.\\
\textit{Parameter:} The integer $k$.\\ \textit{Question:} Does
$\iplan$ have a plan of length at most $k$?
\end{quote}

\smallskip

The propositional version of the \strips\ planning language 
can be treated as the special case of \sasplus\ satisfying 
restriction B.  
More precisely, propositional \strips\ is commonly used in two variants,
differing in whether negative preconditions are allowed or not.
Both these variants as well as \sasplus\ have been shown to be
equivalent under a strong form of polynomial reduction that
preserves solution length~\cite{Backstrom95}.
Hence, we will not treat \strips\ explicitly.
It should be noted, though, that while the equivalence holds in 
the general case, it often breaks down when further restrictions 
are imposed.

\smallskip

\noindent We will mainly consider the following four restrictions, 
originally defined by B\"{a}ckstr\"{o}m and Klein
\cite{BackstromKlein91}.
\begin{quote}
\begin{description}
  \item[Post-unique (P):] For each $v \in \vars$ and each $x \in \dom$ there is at
most one $a \in \acts$ such that $\proj{\eff(a)}{v} = x$.
  \item[Unary (U):] For each $a \in \acts$, $\proj{\eff(a)}{v} \neq \undef$
for exactly one $v \in \vars$.
  \item[Binary (B):] $\card{\dom} = 2$.
  \item[Single-valued (S):] For all $a,b \in \acts$ and all $v \in \vars$, if
$\proj{\pre(a)}{v} \neq \undef$, $\proj{\pre(b)}{v} \neq \undef$ and
$\proj{\eff(a)}{v} = \proj{\eff(b)}{v} = \undef$, then
$\proj{\pre(a)}{v} = \proj{\pre(b)}{v}$.
\end{description}
\end{quote}

\noindent For any set $R$ of such restrictions we write
\mbox{$R$-\BPE} to denote the restriction of \BPE\ to only instances
satisfying the restrictions in $R$.

Additionally we will consider restrictions on the number of
preconditions and effects as previously considered
by Bylander~\cite{Bylander94}. For two non-negative integers 
$p$ and $e$ we write\\
\mbox{$(p,e)$-\BPE} to denote the restriction of \BPE{} to only
instances where every action has at most $p$ preconditions and at most
$e$ effects. 
Apart from doing a parameterized analysis, we also generalize
Bylander's results to \sasplus; all our membership results hold
for arbitrary domain size while all our hardness results apply
already for binary domains.
All non-parameterized hardness results in Table~\ref{table:bylander}
follow directly from 
Bylander's classical complexity results for 
\strips~\cite[Fig.~1 and~2]{Bylander94}. 
Note that we use
results both for bounded and unbounded plan existence,
 which is
justified since the unbounded case is (trivially) polynomial-time
reducible to the bounded case.
The membership results for \PSPACE\ are immediate since
\textsc{Bounded $\SAS$ Planning} is in \PSPACE.
The membership results for \NP\ (when $m_p=0$)
follow from Bylander's~\cite{Bylander94}
Theorem~3.9, which says that every solvable \strips\ instance 
with $m_p=0$ has a plan
of length $\leq m$ where $m$ is the number of actions. 
It is easy to verify that the same bound holds for \sasplus\
instances.

\section{Hardness Results}
\noindent
In this section we prove the three main hardness results of this paper.
For the first proof we need the following problem, which is
\W{2}-complete~\cite[p.~464]{DowneyFellows99}.

\smallskip
\begin{quote}
\noindent
    \textsc{Hitting Set}\\
\noindent    
    \emph{Instance:} A finite set $S$, a collection $C$ of 
    subsets of $S$, and an integer~$k$.\\
\noindent    
    \emph{Parameter:} The integer $k$.\\  
\noindent
    \emph{Question:} Does $C$ have a \emph{hitting set} of cardinality
    at most $k$, i.e., is there a set $H \subseteq S$ with
    $\card{H} \leq k$ and $H \cap c \neq \emptyset$ 
    for every $c \in C$?
\end{quote}

\begin{theorem}\label{arb-hard}
  \BPER{B,S}\ 
  is $\W{2}$-hard, even
  when the actions have no preconditions.
\end{theorem}

\begin{proof}
  We proceed by a parameterized reduction from \textsc{Hitting Set}.
  Let $\insti=\tuple{S,C,k}$ be an instance of this problem. We construct an
  instance $\insti'=\tuple{\iplan,k}$ with $\iplan=\tuple{\vars,\dom,\acts,\init,\goal}$
  of the
  \BPER{B,S}\ 
  problem such that $\insti$ has a
  hitting set of size at most $k$ if and only if there is a plan of
  length at most $k$ for $\insti'$ as follows. 
  Let $\vars=\SB v_c \SM c \in C \SE$,
  let $\dom=\SB 0,1 \SE$ and
  let $\acts = \set{a_s \st s \in S}$ such that
  $\proj{\eff(a_s)}{v_c}=1$ if $s \in c$. 
  We set $\init=\tuple{0,\ldots,0}$ and
  $\goal=\tuple{1,\ldots,1}$.
  Clearly, $\iplan$ is binary (B) and no action has any preconditions.  
  It follows trivially from the latter observation that $\iplan$
  is also single-valued (S)
  It remains to show that $\iplan$ has a plan of
  length at most $k$ if and only if $\insti$ has a hitting set of
  size at most~$k$.

  Suppose that $\insti$ has a hitting set $H=\{h_1,\dotsc,h_l\}$ of size at
  most $k$. Then $\plan \mdef \seq{a_{h_1},\dotsc,a_{h_l}}$ is a plan of length at
  most $k$ for $\iplan$.

  For the reverse direction suppose that there is a plan $\plan=\seq{a_1,\dotsc,
  a_{l}}$ of length at most $k$ for $\iplan$. 
  We will show that the set $H_\iplan \mdef \SB s \SM a_s \in \plan
  \SE$ is a hitting set of size at most $k$ for $\insti$. 
  Since  $\proj{\init}{v_c}=0$ and $\proj{\goal}{v_c}=1$ for every 
  $c \in C$, it follows that
  for every $c \in C$ there has to be an action $a_s$ with $s \in
  c$ in $\plan$. Hence, $H_\iplan$ is a hitting set for $\insti$ and because
  $l \leq k$ it follows that $|H_\iplan|\leq k$.
\end{proof}

We continue with the second result, using the following problem,
which is \W{1}-complete~\cite{Pietrzak03}.

\smallskip
\begin{quote}
  \noindent
  \textsc{Multicolored Clique}\\
  \noindent  
  \emph{Instance:} A $k$-partite graph $G=\tuple{V,E}$ with 
  a partition $V_1,\dots,V_k$ of $V$ such that $|V_1|=\dots=|V_n|=n$.
\\
  \noindent  
  \emph{Parameter:} The integer $k$.\\  
  \noindent
  \emph{Question:} Are there nodes $v_1,\dots,v_k$ such that
  $v_i\in V_i$ for all $1\leq i \leq k$ and
  $\{v_i,v_j\}\in E$ for all $1 \leq i < j \leq k$ 
  (i.e. the subgraph of $G$ induced by $\{v_1,\dots,v_k\}$ is a clique
  of size $k$)?
\end{quote}

\smallskip

\begin{theorem} \label{unary-hard}
  \BPER{U,B,S}\  is $\W{1}$-hard, even for binary instances where 
  every action has at most  $1$ precondition and $1$ effect.
\end{theorem}
\begin{proof}
  We proceed by a  parameterized reduction from \textsc{Multicolored Clique}.  
  Let $G = (V, E)$ be a $k$-partite
  graph with partition $V_1,\dots,V_k$ of $V$.
  Let $k_2 = \frac{k(k-1)}{2} = \binom{k}{2}$
  and $k'=7k_2+k$, and
  define $J_i=\SB j \SM 1 \leq j \leq k \textup{
    and }j\neq i \SE$ for every $1 \leq i \leq k$. 

  For the  \BPER{U,B,S}\ instance $\iplan$ we introduce 
  four kinds of variables:
  \begin{enumerate}
  \item For every $e\in E$ we introduce an \emph{edge variable}
    $x(e)$.
  \item
    For every $1\leq i \leq k$ and $v\in V_i$ we introduce $k - 1$
    \emph{vertex variables} $x(v, j)$ where $j \in J_i$.
  \item
    For every $1\leq i\leq k$ and every $j\in J_i$ we introduce a
    \emph{checking variable} $x(i, j)$.
  \item
    For every $v\in V$, we introduce a
    \emph{clean-up variable} $x(v)$.
  \end{enumerate}

  \noindent 
  We also introduce five kinds of actions:
  \begin{enumerate}
  \item For every $e\in E$ we introduce an action $a^{e}$ such that
    $\proj{\eff(a^{e})}{x(e)} = 1$.
  \item For every $e = \{v_i, v_j\}\in E$ where $v_i \in V_{i}$ and
    $v_j\in V_j$, we introduce two actions $a^{e}_{i}$ and $a^{e}_{j}$ such
    that $\pre(a^{e}_{i})[x(e)] = 1$, $\proj{\eff(a^{e}_{i})}{x(v_i, j)}=1$,
    $\pre(a^{e}_{j})[x(e)] = 1$ and $\proj{\eff(a^{e}_{j})}{x(v_j, i)}=1$.
  \item For every $v \in V_i$ and $j\in J_i$, we introduce 
    an action $a^{v}_{j}$ such that $\pre(a^{v}_{j})[x(v,j)] = 1$ 
    and $\proj{\eff(a^{v}_{j})}{x(i, j)} = 1$.
  \item For every $v \in V$, we introduce an
    action $a_v$ such that $\proj{\eff(a_v)}{x(v)} = 1$.
  \item For every $v\in V_i$, for some $1\leq i \leq k$, and  $j\in J_i$,
    we introduce an action $a^j_v$ such that $\pre(a^j_v)[x(v)] = 1$ and
    $\proj{\eff(a^j_v)}{x(v, j)} = 0$.
  \end{enumerate}
  \noindent
  Let $\acts_1,\ldots,\acts_5$ be sets of actions corresponding 
  to these five groups, 
  and let 
  $\acts = \acts_1 \union \ldots \union \acts_5$
  be the set of all actions.
  Let $\init=\tuple{0,\ldots,0}$ and  define $\goal$ such that
  all checking variables $x(i, j)$ are $1$, all vertex variables
  $x(v, j)$ are $0$ and the rest are $\undef$.
  Clearly $\iplan$ can be constructed from $G$ in polynomial time. 
  Furthermore, $\iplan$ is binary and no action has more than
  $1$ precondition and $1$ effect.
  The theorem will follow after we have shown the
  following claim.
  \begin{claim}
    $G$ has a $k$-clique if and only if $\iplan$ has a plan of length at
    most $k'$.
  \end{claim}

  \noindent ($\Rightarrow$) Assume $G$ has a $k$-clique $K = (V_K, E_K)$
  where $V_K = \{v_1, \ldots, v_k\}$ with $v_i\in V_i$ for every 
  $1\leq i \leq
  k$. We construct a plan $\plan$ for $\iplan$ as follows.
  For all $1\leq i < j \leq k$, we apply the actions
  $a^{\{v_i,v_j\}}\in\OOO_1$, to select the edges of the clique,
  and $a^{{\{v_i,v_j\}}}_{i},a^{{\{v_i,v_j\}}}_{j}\in\OOO_2$,
  to set the corresponding connection information for the vertices
  of the clique.
  This gives $3k_2$ actions.
  Then for each checking variable $x(i, j)$, for every $1\leq i\leq k$ and
  $j\in J_i$, we apply $a^{v_i}_{j}\in\OOO_3$ to verify that
  the selected vertices do form a clique.
  This gives $2k_2$ actions. 
  Now we have all checking variables set to the required value $1$, but the
  vertex variables $x(v_i, j)$, for $1\leq i\leq k$ and $j\in J_i$, still bear
  the value $1$ which will have to be set back to $0$ in the goal state. So we
  need some actions to ``clean up'' the values of these vertex variables.
  First we set up a cleaner for each vertex $v_i$ by applying
  $a_{v_i}\in\OOO_4$. This gives $k$ actions. Then we
  use $a^j_{v_i}\in\OOO_5$ for all $j\in J_i$ to set the vertex variables
  $x(v_i, j)$ to $0$. This requires
  $2k_2$ actions. We observe that all the checking
  variables are now set to $1$, and all the vertex variables are set to $0$. The
  goal state is therefore reached from the initial state by the execution
  of exactly $k'= k + 7k_2$ actions, as required. Hence the
  forward direction of the claim is shown.

  \medskip
  \noindent ($\Leftarrow$) Assume $\plan$ is a plan for $\iplan$ of length at
  most $k'$. In the following, we use $\OOO^\plan_s$ to denote the set of
  actions from $\OOO_s$ that occur in
  the plan $\plan$ and we use $\OOO^\plan$ to denote the set of all
  actions from $\OOO$ that occur in $\plan$.
  In the initial state all variables are set to $0$ and in the
  goal state all the $2k_2$ checking variables must be set to $1$,
  so it follows that $\Card{\OOO_3^\plan} \geq 2k_2$ since each action in
  $\OOO_3$ sets exactly one checking variable to $1$. 
  Each action in $\OOO_3^\plan$ requires that a distinct vertex variable
  is set to $1$ first. This can only be accomplished by the execution of
  an action from $\OOO_{2}$, hence $\Card{\OOO_{2}^\plan}\geq
  2k_2$.
  In turn, to make sure that some action in $\OOO_2^\plan$ can be executed,
  some edge variable must be set to $1$ first by an action in
  $\OOO_1^\plan$. However, one edge variable provides the precondition for
  at most two actions in $\OOO_2^\plan$.  Hence we require
  $\Card{\OOO_1^\plan}\geq k_2$.
  The actions in $\OOO_{2}^\plan$ set at least $2k_2$ vertex
  variables to $1$. In the goal state all vertex variables must have the
  value $0$ again, hence we need to apply at least $2k_2$
  actions from $\OOO_5$, and consequently $\Card{\OOO_5^\plan}\geq 2k_2$.
  In order to apply an action in $\OOO_5^\plan$, we first need to set a clean-up variable
  to $1$ with an action from $\OOO_4^\plan$. One clean-up
  variable provides the precondition for at most $k-1$ actions in
  $\OOO_5^\plan$, hence $\Card{\OOO_4^\plan} \geq k$. 
  In total we get $\Card{\OOO^\plan}\geq \sum_{s=1}^5 \Card{\OOO_s^\plan} \geq
  7k_2+k=k'$. Conversely, $k'$ is an upper bound on the length of $\plan$,
  and the length of $\plan$ is clearly an upper bound on the number of actions in
  $\OOO^\plan$, hence $\Card{\OOO^\plan}\leq k'$. Thus $\Card{\OOO^\plan}=k'$ and
  $\plan$ has exactly length $k'$. It follows that in all the above
  inequalities, equality holds, i.e., we have $\Card{\OOO_1^\plan}=k_2$,
  $\Card{\OOO_2^\plan}= \Card{\OOO_3^\plan}= \Card{\OOO_5^\plan} = 2 k_2$, and
  $\Card{\OOO_4^\plan}=k$.

  We call a variable \emph{active} if its value gets changed
  during the execution of~$\plan$.  All $2k_2$ checking variables are
  active, and by the above considerations, there are exactly
  $k_2=\Card{\OOO_1^\plan}$ active edge variables, $2k_2=\Card{\OOO_2^\plan}$
  active vertex variables, and $2k_2=\Card{\OOO_4^\plan}$ active clean-up
  variables.  We conclude that each active clean-up variable $x(v)$,
  $v\in V_i$, must provide the precondition for actions in $\OOO_5^\plan$
  to set $k-1$ vertex variables to $0$ (these vertex variables are
  active). This is only possible if these vertex variables are exactly
  the $k-1$ variables $x(v,j)$ for $j\in J_i$.  For each $1\leq i \leq
  k$ and $j\in J_i$ the checking variable $x(i,j)$ is active, hence
  there must be some vertex $v\in V_i$ such that the vertex variable
  $x(v,j)$ is active, in order to provide the precondition for the
  action $a^{v}_j \in \OOO_3^\plan$.  We conclude that for each $1\leq i
  \leq k$, the set $V_i$ contains exactly one vertex $v_i$ such that
  $x(v_i)$ and $x(v_i,j)$, $j\in J_i$ are all active.  We show that
  these vertices $v_1,\dots,v_k$ induce a clique in $G$.

  Since we have $k_2$ active edge variables and $2k_2$ active vertex
  variables, each edge variable $x(e)$ must provide the precondition for
  two actions in $\OOO_2^\plan$ that make two vertex variables
  active. This is only possible if $e=\{u,v\}$ for $u\in V_i,v\in V_j$, and
  the two vertex variables are $x(v,j)$ and $x(u,i)$. We conclude that
  the active edge variables are exactly the variables $x(\{v_i,v_j\})$,
  $1\leq i < j \leq k$. Hence, indeed, the vertices $v_1,\dots,v_k$
  induce a clique in~$G$. This concludes the proof of the claim. The
  theorem follows.
\end{proof}

\begin{theorem}\label{the:03-BPE-hard}
  $(0,3)$-\BPE{} is $\W{1}$-hard, even for binary instances.
\end{theorem}
\begin{proof}
  By parameterized reduction from \textsc{Multicolored Clique}.  
  Let $\insti=\tuple{G,k}$ be an instance of this problem where
  $G=\tuple{V,E}$, $V_1,\dots,V_k$ is the partition of $V$, 
  $\Card{V_1}=\dotsb=\Card{V_k}=n$ and parameter $k$.
  We construct a $(0,3)$-\BPE{} instance
  $\insti'=(\iplan,k')$ with
  $\iplan=\tuple{\vars,\dom,\acts,\init,\goal}$ such that
  $\insti$ has a multicolored clique of size at most $k$ if and only if
  $\iplan$ has a plan of length at most $k'$.

  We set $\vars=V(G) \cup \SB p_{i,j} \SM 1
  \leq i < j \leq k \SE$, $\dom=\{0,1\}$, $\init=\tuple{0,\ldots,0}$,
  $\proj{\goal}{p_{i,j}}=1$ for every $1 \leq i < j \leq
  k$ and $\proj{\goal}{v}=0$ for every $v \in V(G)$. 
  Furthermore, the set $\acts$ contains the following actions:
  \begin{itemize}
  \item For every $v \in V(G)$ one action $a_v$ with $\proj{\eff(a_v)}{v}=0$;
  \item For every $e=\{v_i,v_j\} \in E(G)$ with $v_i \in V_i$ and $v_j
    \in V_j$ one action $a_e$ with $\proj{\eff(a_e)}{v_i}=1$,
    $\proj{\eff(a_e)}{v_j}=1$, and $\proj{\eff(a_e)}{p_{i,j}}=1$. 
  \end{itemize}
  Clearly, $\iplan$ is binary and no action in $\acts$ has any
  precondition or more than $3$  effects.  
  The theorem will follow after we have shown the
  following claim.
  \begin{claim}\label{claim:clique}
    $G$ contains a $k$-clique if and only if $\iplan$ has
    a plan of length at most $k'=\binom{k}{2}+k$.
  \end{claim}
  \sloppypar Suppose that $G$
  contains a $k$-clique with vertices $v_1,\dots,v_k$ and edges $e_1,
  \dots, e_{\binom{k}{2}}$.  Then
  $\plan \mdef \seq{a_{e_1},\dots,a_{e_{\binom{k}{2}}},a_{v_1},\dots,a_{v_k}}$ is a
  plan of length $k'$ for~$\iplan$.
  
  For the reverse direction suppose that $\plan$ is a plan of length
  at most $k'$ for $\iplan$. Because $\proj{\init}{p_{i,j}}=0 \neq
  \proj{\goal}{p_{i,j}}=1$ the plan $\plan$ has to contain at least
  one action $a_e$ where $e$ is an edge between a vertex in $V_i$ and
  a vertex in $V_j$ for every $1 \leq i < j \leq k$. Because
  $\proj{\eff(a_{e=\{v_i,v_j\}})}{v_i}=1 \neq \proj{\goal}{v_i}=0$ and
  $\proj{\eff(a_{e=\{v_i,v_j\}})}{v_j}=1 \neq \proj{\goal}{v_j}=0$ for
  every such edge $e$ it follows that $\plan$ has to contain at least
  one action $a_{v}$ with $v \in V_i$ for every $1 \leq i \leq
  k$. Because $k'=\binom{k}{2}+k$ it follows that $\plan$ contains
  exactly $\binom{k}{2}$ actions of the form $a_e$ for some edge $e
  \in E(G)$ and exactly $k$ actions of the form $a_v$ for some vertex
  $v \in V(G)$. It follows that the graph $K=(\SB v \SM a_v \in \plan
  \SE,\SB e \SM a_e \in \plan \SE)$ is a $k$-clique of~$G$.
\end{proof}

\section{Membership Results}
\label{sec:modelcheck}
\noindent
Our membership results are based
on {\em First-Order Logic (FO) Model Checking}.
For a class $\Phi$ of FO formulas  we define the following parameterized
problem.

  \smallskip
  \begin{quote}
  \noindent
  \FOMCR{\Phi}\\
  \noindent
  \emph{Instance:} A finite structure $\AAA$, an FO formula $\phi \in \Phi$.\\
  \noindent
  \emph{Parameter:} The length of $\phi$.\\
  \noindent
  \emph{Question:} Does $\AAA \models \phi$, i.e., is $\AAA$ a model for $\phi$ ?
  \end{quote}
  \smallskip

Let $\Sigma_1$ be the class of all FO formulas of the form
$\exists x_1 \dots \exists x_t . \phi$
where $t$ is arbitrary and $\phi$ is a
quantifier-free FO formula.
For every  positive integer $u$, let
$\Sigma_{2,u}$ denote the class of all
FO formulas of the form
$\exists x_1 \dots \exists x_t \forall y_1 \dots \forall y_u . \phi$
where $t$ is arbitrary and $\phi$ is a quantifier-free FO formula.
The following connections between model checking and parameterized
complexity classes are known.
\begin{proposition}[\mbox{Flum and Grohe~\cite[Theorem 7.22]{FlumGrohe06}}]
\label{pro:fo-model-w1} \label{pro:fo-model-w2} \sloppypar
  The problem \FOMCR{\Sigma_1} is $\W{1}$-complete.
  For every positive integer $u$ the problem \FOMCR{\Sigma_{2,u}} is $\W{2}$-complete.
\end{proposition}
We will reduce \BPE{} to \FOMCR{\Phi}. We start
by defining a relational structure $\AAA(\insti)$ for an arbitrary
\BPE{} instance ${\insti}=\tuple{\iplan,k}$
with $\iplan=\tuple{\vars,\dom,\acts,\init,\goal}$
as follows:
\begin{itemize}
\item The universe of $\AAA(\insti)$ is $\vars \cup \acts \cup \dom^+
  \cup \{\dumop\}$, where $\dumop$ is a novel element that represents a
  ``dummy'' action (which we need for technical reasons).
\item $\AAA(\insti)$ contains the unary relations
  $\varRNS \mdef \vars$, $\opeNS \mdef \acts \cup \{\dumop\}$, $\domRNS \mdef \dom^+$,
  and $\dumopRNS \mdef \{\dumop\}$
  together with the following relations of higher arity:
  \begin{itemize}
  \item $\initRVNS \mdef \SB \tuple{v,\varval} \in \vars
    \times \dom \SM I[v]=\varval \SE$,
  \item $\goalRVNS \mdef \SB \tuple{v,\varval} \in \vars
    \times \dom \SM \goal[v]=\varval \neq \undv \SE$,
  \item $\preRNS \mdef \SB \tuple{a,v} \in \acts \times \vars
    \SM \pre(a)[v] \neq \undv 
    \SE$,
  \item $\postRNS \mdef \SB \tuple{a,v} \in \acts \times \vars
    \SM \proj{\eff(a)}{v} \neq \undv \SE$,
  \item $\preRVNS \mdef 
    \SB \tuple{a,v,\varval} \in \acts\times \vars \times \dom \SM 
      \pre(a)[v] = \varval \neq \undv \SE $
  \item $\postRVNS \mdef \SB \tuple{a,v,\varval} \in \acts
    \times \vars \times \dom \SM 
    \proj{\eff(a)}{v}=\varval \neq \undv \SE$.
  \end{itemize}
\end{itemize}

\begin{theorem}\label{in-W2}
  \BPE\ 
  is in $\W{2}$.
\end{theorem}
\begin{proof}\sloppypar
  We procced by parameterized reduction to the
  problem \FOMCR{\Sigma_{2,2}}, which is $\W{2}$-complete by
  Proposition~\ref{pro:fo-model-w1}. 
  Let
  ${\insti}=\tuple{\iplan,k}$ with $\iplan=\tuple{\vars,\dom,\acts,\init,\goal}$) be an 
  instance of \BPE. We construct an
  instance ${\insti}'=\tuple{\AAA(\insti),\phi}$ of
  \FOMCR{\Sigma_{2,2}} 
  such that ${\insti}$ has a solution if and
  only if ${\insti}'$ has a solution and the length of the formula $\phi$ is
  bounded by some function that only depends on the parameter $k$.
  For the definition of $\phi$ we need the following auxiliary formulas.
  In the following let $0 \leq i \leq k$. We define the formula 
  $\fvalue(\langle a_1,\ldots,a_i \rangle,v,\varval)$, which holds if
  the variable $v$ has value $\varval$
  after applying the actions $a_1,\ldots,a_i$ to the initial state.
  This formula is inductively defined as follows.
  \begin{quote}
    $\fvalue(\langle \rangle,v,\varval) \mdef \initRVNS v, \varval$ 

    $\fvalue(\langle a_1,\ldots,a_i \rangle,v,\varval) \mdef
    (\fvalue(\langle a_1,\ldots,a_{i-1} \rangle,v,\varval) \land \lnot 
    \postR a_i,v )$
 
    \quad $\lor\, \postRV a_i,v,\varval$ 
  \end{quote}
  We also define the formula
  $\fcheckpre(\langle a_1,\ldots,a_i \rangle,v,\varval)$
  which holds if 
  the variable $v$ has the value $\varval$  after 
  the actions $a_1,\dots,a_{i-1}$ have been applied to the initial state,
  whenever $\varval$ is the precondition of the action $a_i$
  on the variable $v$.
  The formula is defined as follows. 
  \begin{quote}
    $\fcheckpre(\langle a_1,\ldots,a_i \rangle,v,\varval) \mdef
\preRV a_i,v,\varval
    \rightarrow \fvalue(\langle a_1,\ldots,a_{i-1} \rangle,v,\varval)$
  \end{quote}
  We further define the formula $\fcheckpreall(\langle a_1,\ldots,a_k
  \rangle,v,\varval)$ which holds if the formula $\fcheckpre(\langle
  a_1,\ldots,a_i \rangle,v,\varval)$ holds for every $0 \leq i \leq k$.
  \begin{quote}
    $\fcheckpreall(\langle a_1,\ldots,a_k \rangle,v,\varval) \mdef
    \bigwedge_{i=1}^k\fcheckpre(\langle a_1,\ldots,a_i \rangle,v,\varval)$
  \end{quote}
  Finally, we define the formula $\fcheckpregoal(\langle a_1,\ldots,a_k
  \rangle,v,\varval)$
  which holds if whenever $\varval$ is the goal
  on the variable $v$, then the variable $v$ has the value $\varval$ after 
  the actions $a_1,\dots,a_{k}$ have been applied to the initial  state.
  The formula is defined as follows.
  \begin{quote}
    $\fcheckpregoal(\langle a_1,\ldots,a_k \rangle,v,\varval) \mdef 
\goalRV v,\varval
    \rightarrow \fvalue(\langle a_1,\ldots,a_k \rangle,v,\varval)$
  \end{quote}

  \medskip\noindent
  We can now define the formula $\phi$ itself as:
  \begin{quote}
$    \phi  \mdef  \exists a_1 \ldots \exists a_k \forall v \forall \varval\,.\,
    (\bigwedge_{i=1}^k\opeNS a_i )\ \land
    (\varR v \land \domR \varval
     \rightarrow $

\hfill $    \fcheckpreall(\langle a_1,\dots,a_k \rangle,v,\varval)  \land 
    \fcheckpregoal(\langle a_1,\dots,a_k \rangle, v,\varval)).
$
%
  \end{quote}
  Evidently $\phi \in \Sigma_{2,2}$, the length of $\phi$ is bounded
  by some function that only depends on $k$ and $\AAA(\insti) \models
  \phi$ if and only if $\iplan$ has a plan of length at most $k$.
  The ``dummy'' action ($\dumop$) guarantees that there is a plan of
  length exactly $k$ whenever there is a plan of length at most $k$.
\end{proof}
Our next results shows that if we restrict ourselves to unary planning
instances then \BPE\ becomes easier (at least from the parameterized
point of view). We show $\W{1}$\hy membership of \BPER{U}\ by reducing 
it to the \textsc{$\Sigma_1$\hy FO Model Checking} problem. The basic
idea  behind the proof is fairly similar to the proof of
Theorem~\ref{in-W2}. However, we
cannot directly express within $\Sigma_1$ that all the preconditions
of an action are satisfied, since we are not allowed to use universal
quantifications within~$\Sigma_{1}$.  Hence, we
avoid the universal quantification with a trick: we observe that the
preconditions only need to be checked with respect to at most $k$
``important'' variables, that is, the variables in which the
preconditions of an action differ from the initial state. Since we
only consider unary planning instances there can be at most $k$
such variables. Hence, it becomes
possible to guess the important variables using only existential
quantifiers.

It remains to check that all the important variables are among these
guessed variables.  We do this without universal quantification by
adding dummy elements $d_1,\dots,d_k$ and a relation $\diffRNS$ to the
relational structure $\AAA(\iplan)$.  The relation associates with
each action exactly $k$ different elements. These elements consist of
all the important variables of the action, say the number of these
variables is $k'$, plus $k-k'$ dummy elements. Hence, by guessing
these $k$ elements and eliminating the dummy elements, the formula
knows all the important variables of the action and can check their
preconditions without using universal quantification.

To accommodate the ``dummy'' elements we start by defining a new
extended structure $\AAA^*(\iplan)$ for an arbitrary
\BPE{} instance ${\insti}=\tuple{\iplan,k}$
with $\iplan=\tuple{\vars,\dom,\acts,\init,\goal}$.
For a partial state $s \in (\dom^+)^{|\vars|}$ and an action $a \in
\acts$ we define the following sets.
\begin{eqnarray*} 
  \diff(s) & \mdef & \SB v \in \VVV \SM \proj{s}{v} \neq
  \undv \textup{ and } \proj{s}{v} \neq \proj{\init}{v} \SE\\ 
  \diff(a) & \mdef & \diff(\pre(a)).
\end{eqnarray*}
We define the structure $\AAA^*(\insti)$ as follows.
\begin{itemize}
\item The universe $A^*$ of $\AAA^*(\insti)$ consists of the elements of
  the universe of $\AAA(\insti)$ plus $k$ novel ``dummy'' elements
  $d_1,\dotsc,d_k$.
\item $\AAA^*(\insti)$ contains all relations of $\AAA(\insti)$ and additionally
  the following relations.
  \begin{itemize}
  \item A unary relation $\dumRNS \mdef \{d_1,\dotsc,d_k\}$.
  \item A binary relation $\diffRNS \mdef \SB \tuple{a,v} \in \acts
    \times \vars \SM a \in \acts \textup{ and }v \in \diff(a) \SE \cup 
    \SB \tuple{a,d_i} \in \acts \times \{d_1,\dotsc,d_k\} \SM a \in \acts \textup{ and } 1
    \leq i \leq k-|\diff(a)| \SE$.
  \item A unary relation $\diffGoalRNS\mdef \SB \tuple{v} \in \vars \SM v \in
    \diff(\goal) \SE \cup \SB \tuple{d_i} \SM 1 \leq i \leq k-|\diff(\goal)| \SE$.
  \end{itemize}
\end{itemize}

Before we show that \BPER{U} is in
$\W{1}$ we need some simple observations about planning.
Let $\iplan=\tuple{\vars,\dom,\acts,\init,\goal}$ be a \PE{} instance,
$\vars' \subseteq \vars$, and $s \in (\dom^+)^{|\vars|}$. 
We denote by $s|\vars'$ the state $s$ restricted to the
variables in $\vars'$ and by $\acts|\vars'$ the set of actions
obtained from the actions in $\acts$ after restricting the
preconditions and effects of every such action to the variables in $\vars'$.
Furthermore, we denote by
$\iplan|\vars'$ the \PE{} instance 
$\tuple{\vars', \dom, \acts|\vars', \init|\vars',\goal|\vars'}$.
\begin{proposition}\label{pro:checkX}
  Let $\plan=\seq{a_1,\dotsc,a_l}$ be a sequence of actions from
  $\acts$. Then $\plan$ is a plan for $\iplan$ if and only if $\plan$
  is a plan for $\iplan|\vars_0$ where 
  $\vars_0=\bigcup_{i=1}^l\diff(a_i) \cup \diff(\goal) \cup \SB v \in \vars \SM
  \proj{\eff(a_i)}{v} \neq \undv \textup{ and }1 \leq i \leq l \SE$
\end{proposition}
\begin{proof}
  If $\plan$ is a plan for $\iplan$, then $\plan$ is a plan
  for $\iplan|\vars'$ whenever $\vars' \subseteq \vars$. In
  particular, $\plan$ is a plan for $\iplan|\vars_0$.

  For the reverse direction assume for a contradiction that
  $\plan|\vars_0$ is a plan for $\iplan|\vars_0$ but
  $\plan$ is not a plan for $\iplan$.
  There are two possible reasons for this: 
  (1) a precondition of some action $a_i$ for $1 \leq i \leq l$,
  is not met or (2) the goal state is not
  reached after having completed the plan. 

  In the first case, consider the state $s$ that $a_i$ is applied in.
  Then there is a variable $v \in \vars$ such that
  $\proj{s}{v} \neq \proj{\pre(a_i)}{v} \neq \undv$. 
  There are two cases to consider:
  \begin{enumerate}
  \item
    $\proj{\init}{v}=\varval \neq \undv$ and
    $\proj{pre(a_i)}{v}=\varval$. 
    In this case an action $a_j$ for some $j < i$ has changed the
    variable $v$ and
    $v \in \SB v \in \vars \SM
    \proj{\eff(a_i)}{v} \neq \undv \textup{ and }1 \leq i \leq l \SE$.
  \item
    $\proj{\init}{v}=\varval$, $\proj{\pre(a_i)}{v}=\varval' \neq
    \undv$, and $\varval \neq \varval'$. This implies
    that $v \in \diff(a_i)$.
  \end{enumerate}
  Hence, in both cases we obtain that $v \in \vars_0$ and that 
  $\plan$ is a plan for $\iplan$.

  In the second case, consider the state $g$ after applying $\plan$ to the initial state.
  Then there is a variable $v \in \vars$ such that
  $\proj{s}{v} \neq \proj{\goal}{v} \neq \undv$. 
  There are two cases to consider:
  \begin{enumerate}
  \item
    $\proj{\init}{v}=\varval \neq \undv$ and
    $\proj{\goal}{v}=\varval$. 
    In this case an action $a_j$ for some $1 \leq j \leq l$ has changed the
    variable $v$ and
    $v \in \SB v \in \vars \SM
    \proj{\eff(a_i)}{v} \neq \undv \textup{ and }1 \leq i \leq l \SE$.
  \item
    $\proj{\init}{v}=\varval$, $\proj{\goal}{v}=\varval' \neq
    \undv$, and $\varval \neq \varval'$. This implies
    that $v \in \diff(\goal)$.
  \end{enumerate}
  Hence, in both cases we obtain that $v \in \vars_0$ and that
  $\plan$ is a plan for $\iplan$.
\end{proof}
\begin{proposition}\label{pro:diffinpost}
  Let $\plan=\seq{a_1,\dotsc,a_l}$ be a plan for $\iplan$.
  Then 
  $\bigcup_{i=1}^l\diff(a_i) \cup \diff(\goal) \subseteq \SB v \in \vars \SM
  \proj{\eff(a_i)}{v} \neq \undv \textup{ and }1 \leq i \leq l \SE$.
\end{proposition}
\begin{proof}
  First assume that $v \in \diff(\goal)$, i.e. $v \in \{v \in \vars \SM \proj{\goal}{v} \neq \undv
  \textup{ and } \proj{\goal}{v} \neq \proj{\init}{v}\}$. Since
  $\plan$ is a plan for~$\iplan$ it follows that there is an action
  $a_j$ in $\plan$, which changes
  the value of $v$, as required.

  Then assume that $v \in \diff(a_j)$ for some $1 \leq j \leq l$, i.e.,
  $v \in \{v \in \vars \SM \proj{\pre(a_j)}{v} \neq \undv \textup{ and
  } \proj{\pre(a_j)}{v} \neq \proj{\init}{v}\}$. Again, since
  $\plan$ is a plan it follows that there is an action $a_k$, $k < j$
  in $\plan$, which changes the value of $v$, as required.
\end{proof}
\begin{corollary}\label{cor:restriction-plan}
  Let $\plan=\seq{a_1,\dotsc,a_l}$ be a sequence of actions from
  $\acts$
  and $\vars_0=\SB v \in \vars \SM \proj{\eff(a_i)}{v} \neq \undv \textup{ and }1
  \leq i \leq l \SE$. Then $\plan$ is a plan for $\iplan$ if and
  only if $\bigcup_{i=1}^l\diff(a_i) \cup \diff(\goal) \subseteq \vars_0$ and
  $\plan$ is a plan for $\iplan|\vars_0$.
\end{corollary}

\begin{theorem}\label{in-W1}
  \BPER{U}\ is in $\W{1}$.
\end{theorem}
\begin{proof}
  We proceed by a parameterized reduction to the $\W{1}$-complete 
  problem \FOMCR{\Sigma_1}. Let
  ${\insti}=\tuple{\iplan,k}$ with $\iplan=\tuple{\vars,\dom,\acts,\init,\goal}$ be an 
  instance of \BPER{U}. We construct an
  instance ${\insti}'=\tuple{\AAA^*(\insti),\phi}$ of
  \FOMCR{\Sigma_1} 
  such that ${\insti}$ has a solution if and
  only if ${\insti}'$ has a solution and the length of the formula $\phi$ is
  bounded by some function that only depends on the parameter~$k$.

  The formula $\phi$ uses the following
  existentially quantified variables:
  \begin{itemize}
  \item The variables $a_1,\dotsc,a_k$. The values of these variables correspond to the
    at most $k$ actions of a plan for $\iplan$.
  \item The variables $v_1,\dotsc,v_k$. The values of these variables correspond to the
    variables that are involved in the effects of the
    actions assigned to $a_1,\dotsc,a_k$,
    i.e., it holds that $\proj{\eff(a_i)}{v_i} \neq \undv$ for every $1 \leq
    i \leq k$. Because $\iplan$ is unary there is at most one such
    variable for each of the at most $k$ actions in a potential plan for $\iplan$.
  \item The variables $d_1,\dotsc,d_k$. These variables are so-called
    ``dummy'' variables that we use to check the maximality of certain
    sets.
  \item The variables $\varval_{1,1},\dotsc, \varval_{1,k}, \dotsc,
    \varval_{k,1}, \dotsc,\varval_{k,k}$. These variables are
    used to check the preconditions of the
    actions $a_1,\dotsc,a_k$. Here $\varval_{i,j}$ represents 
    $\proj{\pre(a_i)}{v_j}$ for every $1 \leq i,j \leq k$.
  \item The variables $\varval_{g,1},\dotsc,\varval_{g,k}$. These
    variables are
    used to check whether all conditions of the goal state are
    met after the actions $a_1,\dotsc,a_k$ have been executed on the
    initial state. Here $\varval_{g,i}$ represents $\proj{\goal}{v_i}$.
  \end{itemize}
  We define $\phi$ in such a way that $\AAA^*(\insti) \models \phi$ if
  and only if there is a sequence of actions $a_1,\dotsc,a_k$ and a
  set $\vars_0=\SB v \in \vars \SM \proj{\eff(a_i)}{v} \neq \undv
  \textup{ and } 1 \leq i \leq k \SE$ of variables with
  $\bigcup_{i=1}^k\diff(a_i) \cup \diff(\goal) \subseteq \vars_0$ such that
  $a_1,\dotsc,a_k$ is a plan for $\iplan|\vars_0$. Because of
  Corollary~\ref{cor:restriction-plan} it then
  follows that $\AAA^*(\insti) \models \phi$ if and only if there
  is a plan of length at most $k$ for $\iplan$.
  Consequently, the formula $\phi$ has to ensure the following
  properties:
  \begin{enumerate}
  \item[P1] For every $1 \leq i \leq k$ if the variable $a_i$ is
    assigned to an action other than the ``dummy'' action ($\dumop$),
    then the variable $v_i$ is assigned to the unique variable $v \in
    \vars$ with $\proj{\eff(a_i)}{v} \neq \undv$. In the following we
    denote by $\vars_0$ the variables in $\vars$ that are assigned to
    the variables $v_1,\dotsc,v_k$.
  \item[P2] For every $1 \leq i \leq k$ it holds that $\diff(a_i)
    \subseteq \vars_0$.
  \item[P3] $\diff(\goal) \subseteq \vars_0$.
  \item[P4] For every $1 \leq
    i \leq k$ all preconditions of the action $a_i$ on the
    variables $v_1,\dotsc,v_k$ are met after the execution of the
    actions $a_1,\dotsc,a_{i-1}$ on the initial state.
  \item[P5]
    The goal state on the variables
    $v_1,\dotsc,v_k$ is reached after the
    execution of the actions $a_1,\dotsc,a_k$ on the initial state.
  \end{enumerate}
  Observe that the properties P1--P3 ensure that the variables
  $v_1,\dotsc,v_k$ are assigned to a set of variables $\vars_0$ with 
  $\vars_0=\SB v \in \vars \SM \proj{\eff(a_i)}{v} \neq \undv \textup{
    and } 1 \leq i \leq k \SE$ (or to the ``dummy'' action)
  and 
  $\bigcup_{i=1}^k\diff(a_i) \cup \diff(\goal) \subseteq \vars_0$. The
  properties P4--P5 make sure that the sequence of actions
  $\seq{a_1,\dotsc,a_k}$ is a plan for $\iplan|\vars_0$.

  \medskip\noindent 
  \sloppypar
  The formula $\phi$ is composed of several auxiliary formulas that we
  define next.
  We define a formula $\fcheckpost(a_1, \dotsc, a_k,v_1, \dotsc v_k, \varval_1,
  \dotsc,\varval_k)$ that ensures property P1, i.e., $\proj{\eff(a_i)}{v_i}
  =x_i$ or $a_i=\dumop$ for every $1 \leq i \leq k$. 
  \begin{quote}
    $\fcheckpost(a_1, \dotsc, a_k,v_1, \dotsc v_k, \varval_1,
    \dotsc,\varval_k) \mdef \bigwedge_{i=1}^k(\postRV a_iv_i\varval_i
    \lor \dumopR a_i )$
  \end{quote}
  For every $1 \leq i \leq k$ we define a formula $\fdiffop(a_i, v_1,
  \dotsc, v_k, d_1,\dotsc,d_k)$ that holds if $\diff(a_i) \subseteq
  \vars_0$. To check this, the formula checks that all of the exactly $k$
  tuples in $\diffR$ that contain $a_i$ (recall the definition of the
  relation $\diffR$ in the structure $\AAA^*(\iplan)$) are tuples of the
  form $(a_i,v_j)$ or $(a_i,d_j)$ for some $1 \leq j \leq k$.
  \begin{quote}
    $ \fdiffop(a_i,v_1,\dotsc,v_k,d_1,\dotsc,d_k) \mdef   \dumopR a_i \lor$

 \hfill  $
    \bigvee_{J \subseteq \{1,\dots,k\}}
     \big[\bigwedge_{j\neq j' \in J} v_j \neq v_{j'} \land
    \bigwedge_{j\in J}   \diffR a_i v_j \land 
    \bigwedge_{1 \leq j \leq
      k-\Card{J}}  \diffR a_i d_j\big] $
  \end{quote} 
  \medskip\noindent We define a formula that ensures property P2.
  \begin{quote}
    $ \fdiffopall(a_1, \dotsc, a_k,v_1,\dotsc,v_k,d_1,\dotsc,d_k) \mdef
     \bigwedge_{i=1}^{k}\fdiffop(a_i,v_1,\dotsc,v_k,d_1,\dotsc,d_k) $
  \end{quote}
  \medskip\noindent
  Similarly to $ \fdiffop$ above we define a formula $\fdiffgoal(v_1,
  \dotsc, v_k, d_1,\dotsc,d_k)$ that ensures property P3, i.e.,
  $\diff(\goal) \subseteq \vars_0$.
  \begin{quote}
    $ \fdiffgoal(v_1,\dotsc,v_k,d_1,\dotsc,d_k) \mdef$

$\hfill  \bigvee_{J \subseteq \{1,\dots,k\}}
    \; \big[\bigwedge_{j\neq j' \in J} v_j \neq v_{j'} \land
    \bigwedge_{j\in J}   \diffGoalR  v_j \land 
     \bigwedge_{1 \leq j \leq
      k-\Card{J}}  \diffGoalR  d_j\big] $
  \end{quote} 

  \medskip\noindent
  For every $1 \leq i \leq k$, we define a formula $\fvalue(a_1, \dotsc, a_i, v,\varval)$ that holds if 
  the variable~$v$
  has value $\varval$ after the actions $a_1,\dotsc,a_i$ have
  been executed on the initial state. We define the formulas inductively as
  follows.
  \begin{quote}
    $ \fvalue(v,\varval) \mdef \initR v\varval$

    $\fvalue(a_1, \dotsc,a_i,v,\varval) \mdef
    (\fvalue(a_1,\dotsc,a_{i-1},v,\varval) \land \lnot \postR a_iv )
    \lor \postRV a_iv\varval$
  \end{quote}
  \medskip\noindent For every $1 \leq i \leq k$ we define a formula
  $\fcheckpre(a_1,\dotsc,a_i,v_1,\dotsc,v_k,\varval_1,\dotsc,\varval_k)$
  that holds if
  all preconditions of the action $a_i$
  defined on the variables $v_1,\dotsc,v_k$ are met after the
  actions $a_1,\dotsc,a_{i-1}$ have been executed on the initial state.
  \begin{quote}
    $\fcheckpre(a_1, \dotsc,a_i,v_1,\dotsc,v_k, \varval_1, \dotsc,
    \varval_k) \mdef\\
    \myindent    (\bigwedge_{j=1}^{k}(\preRV a_iv_j\varval_j \land
    \fvalue(a_1,\dotsc,a_{i-1},v_j,\varval_j)) \lor \lnot \preR
    a_iv_j)$
  \end{quote} 

  \medskip\noindent
  We define a formula
  $\fcheckpreall(a_1,\dotsc,a_k,v_1,\dotsc,v_k,\varval_{1,1},\dotsc,\varval_{k,k})$
  that ensures property P4, i.e.,
  $\fcheckpre(a_1,\dotsc,a_i,v_1,\dotsc,v_k,\varval_{i,1},\dotsc,\varval_{i,k})$
  for every $1 \leq i \leq k$. 
  \begin{quote} 
    $\fcheckpreall(a_1, \dotsc,a_k,v_1,\dotsc,v_k, \varval_{1,1},
    \dotsc, \varval_{k,k}) \mdef \\
    \myindent
    \bigwedge_{i=1}^{k}\fcheckpre(a_1,\dotsc,a_i,v_1,\dotsc,v_k,\varval_{i,1},\dotsc,\varval_{i,k})$
  \end{quote}

  \medskip\noindent Finally, we define a formula
  $\fcheckpregoal(a_1,\dotsc,a_k,v_1,\dotsc,v_k,\varval_{g,1},\dotsc,\varval_{g,k})$
  that ensure property P5, i.e., all conditions of the goal state on the
  variables $v_1,\dotsc,v_k$ are met after the actions
  $a_1,\dotsc,a_k$ have been executed on the initial state.
  \begin{quote}
    $\fcheckpregoal(a_1, \dotsc,a_k,v_1,\dotsc,v_k, \varval_{g,1}, \dotsc,
    \varval_{g,k}) \mdef \\
    \myindent  (\bigwedge_{i=1}^k (\goalRV v_i\varval_{g,i} \land
    \fvalue(a_1,\dotsc,a_k,v_i,\varval_{g,i})) \lor \lnot \goalR v_i) $
  \end{quote}

  \medskip\noindent Now we can use the above formulas to define the
  required formula $\phi$.  The formula starts with existential
  quantifiers over the variables $a_i,v_i,d_i$, $\varval_{i,j}$, and $x_{g,i}$
  for every $i,j \in \{1,\dots, k\}$, followed by the
  conjunction of the following quantifier-free formulas:
  \begin{itemize}
  \item $\bigwedge_{i=1}^k\opeNS a_i \land 
    \bigwedge_{i=1}^k\varR v_i \land 
    \bigwedge_{i=1}^k\domR \varval_i \land 
    \bigwedge_{1 \leq i,j \leq k}\domR \varval_{i,j} \land$

    $\bigwedge_{i=1}^k \dumR d_i \land\bigwedge_{1 \leq i< j \leq k} d_i\neq d_j$
  \item $\fcheckpost(a_1,\dotsc,a_k,v_1, \dotsc, v_k, \varval_1, \dotsc,
    \varval_k)$
  \item $ \fdiffopall(a_1, \dotsc, a_k,v_1,\dotsc,v_k,d_1,\dotsc,d_k)$
  \item $\fdiffgoal(v_1,\dotsc,v_k,d_1,\dotsc,d_k)$
  \item $\fcheckpreall(a_1,\dotsc,a_k,v_1,\dotsc,v_k,\varval_{1,1},
    \dotsc, \varval_{k,k})$ 
  \item   $\fcheckpregoal(a_1, \dotsc,a_k, v_1,\dotsc, v_k, \varval_{g,1},
    \dotsc, \varval_{g,k})$
  \end{itemize}
  Evidently $\phi \in \Sigma_{1}$, the length of $\phi$ is bounded by some
  function that only depends on $k$, and $\AAA^*(\insti) \models \phi$ if and only
  if $\iplan|\vars_0$ (and hence $\iplan$ according to
  Corollary~\ref{cor:restriction-plan}) has a plan of length at most $k$.
\end{proof}

\section{Fixed-Parameter Tractability}
\noindent
In this section we show that \BPER{P} and $(0,2)$-\BPE{} are
fixed-parameter tractable.

\subsection{\BPER{P}}
\label{sec:partialplan}

\newcommand{\ins}{\text{\normalfont insert}}

This subsection is devoted to a proof of the following theorem.
\begin{theorem}\label{the:Pfpt}
  \BPER{P} is fixed-parameter tractable.
\end{theorem}

Let $\insti=(\iplan,k)$ with $\iplan=(\vars,\acts,\init,\goal)$, be an
instance of \BPER{P}. First, we introduce some terminology
on sequences of actions.  For $l \geq 0$, let
$\plan=\seq{a_1,\dotsc,a_l}$ be a sequence of actions from $\acts$.  We define
$\ins(i,a,\plan)=\seq{a_1,\dotsc,a_{i-1},a,a_i\dotsc,a_l}$.  Let
$\tuple{v,\varval}\in \vars \times \dom$. For $0\leq i < j \leq l+1$ we say
that $\tuple{v,\varval}$ is \emph{required in $\plan$ between positions $i$
  and $j$} if the following two conditions hold:
\begin{enumerate}
\item Either $j\leq l$ and the $j$-th element of $\plan$ is an action
  $a$ with $\proj{\pre(a)}{v}=\varval$, or $j=l+1$ and
  $\proj{\goal}{v}=\varval$.
\item $i$ is the smallest integer such that $\proj{s}{v}\neq \varval$ 
  where $s$ is the state obtained after applying the actions $a_1,
  \dotsc, a_i$ to the initial state (with $i=0$ if $\proj{\init}{v}\neq \varval$).
\end{enumerate}
If $(v,\varval)$ is required in $\plan$ between positions $i$ and $j$,
and $a$ is an action which sets $v$ to $\varval$, then we also say that $a$
is required in $\plan$ (between positions $i$ and $j$).
Note that there can be at most one such action since $\iplan$
is post-unique.

The following claim is immediate from the above definitions.
\begin{claim}\label{claim:fpt1}
  $\plan$ is a plan for $\iplan$ if and
  only if there is no pair $\tuple{v,\varval}\in \vars \times \dom$ which is
  required in $\plan$ between some positions $0\leq i < j \leq l+1$.
\end{claim}
\begin{claim}\label{claim:fpt2}
  If $\plan$ is a subsequence of some plan $\plan^*$ for $\iplan$,
  and $\tuple{v,\varval}$ is required in $\plan$ between
  some positions $i$ and $j$ with $0 \leq i < j \leq l+1$. Then the following holds:
  \begin{enumerate}
  \item $\plan^*$ must contain the unique action $a$ with $\proj{\eff(a)}{v}=\varval$.
  \item for some $i\leq m \leq j$, the sequence
    $\ins(m,a,\omega)$ is a subsequence of $\plan^*$.
  \end{enumerate}
\end{claim}
\begin{proof}
  If $j\leq l$, then clearly without an
  action $a$ that sets $v$ to $\varval$ we cannot meet the
  precondition of the $j$-th action of $\plan$. Similarly, if $j=l+1$,
  then without an action $a$ that
  sets $v$ to $\varval$ we cannot reach the goal state.  Since $\iplan$
  is post-unique, there is at most one such  action $a$, hence the
  first statement of the claim follows. We further observe that
  $\plan^*$ can be obtained from $\plan$ by inserting actions, and one
  of these inserted actions must be $a$, hence the second statement of
  the claim follows.
\end{proof}
The above considerations suggest that we can find a plan by starting
with the empty sequence, and as long as there is a required action,
guessing its position and insert it into the sequence. Next we describe
an algorithm that follows this general idea. It constructs a search
tree, where every node of the tree is labeled with a sequence of
actions. Each leaf of the tree is marked either as a ``success node''
if its label is a plan of length at most $k$, or as a ``failure node'' if
its label is not a subsequence of a plan of length at most $k$.  The
algorithm not only decides whether there exists a plan of length at most
$k$, but it even lists all minimal plans of length at most $k$ (a plan
is called \emph{minimal} if none of its proper subsequences is a plan).

We start with a trivial tree consisting of just a root, labeled with the
empty sequence, and recursively extend this tree. Assume that $T$ is the
search tree constructed so far. Consider a leaf $n$ of $T$ labeled with
a sequence $\plan$ of length $l\leq k$. If $\plan$ is a plan we can
mark $n$ as a success node, and we do not need to extend the search tree
below $n$.  Otherwise, by Claim~\ref{claim:fpt1}, there is some
$\tuple{v,\varval}\in \vars\times \dom$ and some $1\leq i<j \leq l+1$ such
that $\tuple{v,\varval}$ is required in $\plan$ between positions $i$ and $j$
(clearly we can find $v,\varval,i,j$ in polynomial time). Hence, some
action needs to be added to make $\plan$ a plan.  If $l=k$ or if
$\acts$ does not contain an action which sets $v$ to $\varval$, we know
that $\plan$ is not a subsequence of any plan of length at most~$k$,
and we can mark $n$ as a failure node. We do not need to extend the
search tree below $n$.  It remains to consider the case where $l<k$ and
$\acts$ contains an action $a$ which sets $v$ to $\varval$. Since $\iplan$
is post-unique, there is exactly one such~$a$.  We need to insert $a$
into $\plan$ between the positions $i$ and $j$, but we don't know
where. However, there are only $j-i+1\leq k$ possibilities. Therefore we
add below $n$ a child $n_m$ for each possibility $m\in \{i,\dots,j\}$,
and we label $n_m$ with the sequence
$\ins(m,a,\plan)$. Eventually we arrive at a search tree $T$
where all its leaves are marked either as success or failure nodes. The
depth of $T$ is at most $k$, since each node of $T$ of depth $d$ is
labeled with a sequence of length $d$, and we do not add nodes with
sequences of length greater than $k$. Each node has at most $k$
children. Hence $T$ has $O(k^k)$ many nodes. As the time required for
each node is polynomial, building the search tree is fixed-parameter
tractable for parameter $k$.

\begin{claim}\label{claim:fpt3}
  Let $\plan$ be a minimal plan of length $l\leq k$. Then for each  $0
  \leq d \leq l$ the tree $T$ has a node $n_d$ of
  depth $d$ such that the label of $n_d$ is a subsequence of $\plan$. 
\end{claim}
\begin{proof}
  We show the claim by induction on $d$.  The claim is evidently true for
  $d=0$. Let $d>0$ and assume the claim holds for $d-1$. Consequently,
  there is a tree node $n_{d-1}$ at depth $d-1$ which is labeled with a
  subsequence $\plan'$ of $\plan$. Since $\plan'$ is a proper
  subsequence of~$\plan$, and since $\plan$ is assumed to be a minimal
  plan, $\plan'$ is not a plan; thus $n_{d-1}$ is not a success
  node. Since $\plan'$ is a subsequence of~$\plan$, it is not a failure
  node either. Hence $n_{d-1}$ must have children. Consequently there is some
  pair $\tuple{v,\varval}\in \vars \times \dom$ which is required in $\plan'$
  between some positions $0\leq i < j \leq d$, and $n_{d-1}$ has $j-i+1$
  children, each labeled with a sequence $\ins(m,a,\plan')$,
  where $a$ is the unique action from $\acts$ that sets $v$ to~$\varval$.
  By Claim~\ref{claim:fpt2}, at least one of the children of $n_{d-1}$ is
  labeled with a subsequence of $\plan$, hence the induction step holds
  true, and Claim~\ref{claim:fpt3} follows.
\end{proof}

Claim~\ref{claim:fpt3} entails as the special case $d=l$ that $\plan$
appears as the label of a success node. We conclude that each minimal
plan of $\iplan$ of length at most $k$ appears as the label of some success
node of $T$.
Hence, once we have constructed the search tree $T$, we can list all
minimal plans of length at most $k$. In particular, we can decide
whether there exists a plan of length at most $k$ and
Theorem~\ref{the:Pfpt} follows.

In is interesting to note that the same result can be obtained by a
slight adaption of the standard partial-ordering planning algorithm by
McAllester and Rosenblitt~\cite{McAllesterRosenblitt91}.
This suggests that many successful applications of planning
might be cases where the problem is ``almost tractable'' and
the algorithm used happens to implicitly exploit this.
The details of how to modify this algorithm to obtain an FPT
algorithm for \BPER{P} can be found in 
one of our previous papers~\cite{BackstromChenJonssonOrdyniakSzeider12}.

\subsection{$(0,2)$-\BPE{}}\label{sec:res-fpt}

Before we show that $(0,2)$-\BPE{} is fixed-parameter tractable we
need to introduce some notions and prove some simple properties of
$(0,2)$-\BPE{}.
Let $\iplan=\tuple{\vars,\dom,\acts,\init,\goal}$ be an instance of \BPE{}.
We say an action
$a \in \acts$ has an effect on some variable $v \in \vars$ if
$\proj{\eff(a)}{v}\neq \undv$. We call this effect \emph{good} if
furthermore $\proj{\eff(a)}{v}=\proj{\goal}{v}$ or
$\proj{\goal}{v}=\undv$ and we call the effect \emph{bad}
otherwise. We say an action $a \in \acts$ is \emph{good} if it has
only good effects, \emph{bad} if it has only bad effects, and
\emph{mixed} if it has at least one good and at least one bad effect.
Note that if a valid plan contains a bad action then this
action can always be removed without changing the validity of the plan.
Consequently, we only need to
consider good and mixed actions.
Furthermore, we write $\Delta(\vars)$ to denote the set of
variables $v \in \vars$ such that $\proj{\goal}{v}\neq\undv$ and 
$\proj{\init}{v} \neq \proj{\goal}{v}$. 
 
The next lemma shows that we do not need to consider good actions with
more than~$1$ effect for $(0,2)$-\BPE{}. 
\begin{lemma}\label{lem:nodoublegoodactions}
There is a parameterized reduction from $(0,2)$-\BPE{} to
$(0,2)$-\BPE{} that maps an instance  $\insti=\tuple{\iplan,k}$ to an instance
  $\insti'=\tuple{\iplan',k'}$ where $k'=k(k+3)+1$ and
  no good action of $\insti'$ affects more than one variable.
  Furthermore, if $\iplan$ is binary, then also $\iplan'$ is binary.
\end{lemma}
\begin{proof}
  The required instance $\insti'$ is constructed from $\insti$ as follows. 
  $\vars'$ contains the following variables:
  \begin{itemize}
  \item All variables in $\vars$;
  \item One binary variable $g$;
  \item For every action $a \in \acts$ and every $1 \leq i \leq
    k+2$ one binary variable $v_i(a)$;
  \end{itemize}
  $\acts'$ contains the following actions:
  \begin{itemize}
  \item For every mixed action $a \in \acts$ that has a good effect on
    the variable $v$ and a bad effect on the variable $v'$,
    there is 
    \begin{itemize}
    \item 
      one action  $a_1(a)$ such that
    $\proj{\eff(a_1(a))}{v'}=\proj{\eff(a)}{v'}$ and \\
    $\proj{\eff(a_1(a))}{v_1(a)}=0$, 
    \item
      one action $a_i(a)$ for all $1 < i < k+3$ such that
    $\proj{\eff(a_i(a))}{v_{i-1}(a)}=1$ \\
      and $\proj{\eff(a_i(a))}{v_i(a)}=0$, as well as 
    \item
      one action $a_{k+3}(a)$ such
    that $\proj{\eff(a_{k+3}(a))}{v_{k+2}(a)}=1$ and \\
    $\proj{\eff(a_{k+3}(a))}{v}=\proj{\eff(a)}{v}$;
    \end{itemize}
  \item  For every good action $a \in \acts$ that has only one effect
    on the variable $v$, there is 
    \begin{itemize}
    \item 
      one action $a_1$ such that
    $\proj{\eff(a_1(a))}{g}=1$ and \\
    $\proj{\eff(a_1(a))}{v_1(a)}=0$, 
    \item
      one action $a_i(a)$ for all $1 < i < k+3$ such that
    $\proj{\eff(a_i(a))}{v_{i-1}(a)}=1$ \\ and 
    $\proj{\eff(a_i(a))}{v_i(a)}=0$, as well as 
    \item
      one action $a_{k+3}(a)$ such that 
    $\proj{\eff(a_{k+3}(a))}{v_{k+2}(a)}=1$ and \\
    $\proj{\eff(a_{k+3}(a))}{v}=\proj{\eff(a)}{v}$;
    \end{itemize}

  \item For every good action $a \in \acts$ that has two effects
    on the variables $v$ and $v'$, there is 
    \begin{itemize}
    \item 
    one action
    $a_1(a)$ such that
    $\proj{\eff(a_1(a))}{g}=1$ and
    $\proj{\eff(a_1(a))}{v_1(a)}=0$, 
    \item
      one action $a_i(a)$ for all $1 < i <  k+2$ such that
    $\proj{\eff(a_i(a))}{v_{i-1}(a)}=1$ \\ and 
    $\proj{\eff(a_i(a))}{v_i(a)}=0$, 
    \item
      one action $a_{k+2}(a)$ such
    that $\proj{\eff(a_{k+2}(a))}{v_{k+1}(a)}=1$ and \\
    $\proj{\eff(a_{k+2}(a))}{v}=\proj{\eff(a)}{v}$, as well as 
    \item
      one action
    $a_{k+3}(a)$ such
    that $\proj{\eff(a_{k+3}(a))}{v_{k+1}(a)}=1$ and \\
    $\proj{\eff(a_{k+3}(a))}{v'}=\proj{\eff(a)}{v'}$;
    \end{itemize}
  \item One action $a_g$ with $\proj{\eff(a_g)}{g}=0$.
  \end{itemize}
  We set $\dom'=\dom \cup \{0,1\}$, $\proj{\init'}{v}=\proj{\init}{v}$
  for every $v \in \vars$, $\proj{\init'}{v}=0$ for every $v
  \in \vars' \setminus \vars$, $\proj{\goal'}{v}=\proj{\goal}{v}$
  for every $v \in \vars$, $\proj{\goal'}{v}=0$ for every $v
  \in \vars' \setminus \vars$, and $k'=k(k+2)+1$.

  The main idea of the reduction is to replace every action $a$ in
  $\iplan$ with $k+3$ new actions that form a chain that has the
  same effect as $a$ if all its actions are included in the plan.
  (For an action with two good effects, it is possible to achieve
  only one of these by including $k+2$ actions).
  Clearly, $\insti'$ can be constructed from $\insti$ by an algorithm
  that is fixed-parameter tractable (with respect to $k$) and 
  $\insti'$ is an
  instance of $(0,2)$-\BPE{} where no good
  action affects more than $1$ variable. 
  Furthermore, all new variable domains introduced are binary so
  $\iplan'$ has the same maximum domain size as $\iplan$.
  It remains to show that
  $\insti'$ is equivalent to $\insti$.

  Suppose that $\plan=\seq{a_1,\dotsc,a_l}$ is a plan of length at
  most $k$ for $\iplan$. Then\\
  $\seq{a_{k+3}(a_1),\dotsc,a_{1}(a_1),\dotsc,
    a_{k+3}(a_l),\dotsc,a_{1}(a_l),a_g}$ is a plan of length $l(k+3)+1\leq
  k(k+3)+1$ for $\iplan'$.

  To see the reverse direction suppose that
  $\plan'=\seq{a_1,\dots,a_{l'}}$ is a minimal (with respect to
  sub sequences) plan of length at most $k'$ for $\iplan'$. We say that
  $\plan'$ \emph{uses} an action $a \in \acts$ if $a_i(a) \in \plan'$
  for some $1 \leq i \leq k+3$. We also define an order of the actions
  used by $\plan'$ in the natural way, i.e., for two actions $a,a' \in
  \acts$ that are used by $\plan'$ we say that $a$ is smaller than $a'$
  if the first occurrence of an action $a_i(a)$ (for some $1 \leq i
  \leq k+3$) in $\plan'$ 
  is before the first occurrence of an action
  $a_i(a')$ (for some $1 \leq i \leq k+3$) in $\plan'$.
  
  Let $\plan=\seq{a_1,\dotsc,a_{l}}$ be the (unique) sequence
  of actions in $\acts$ that are
  used by $\plan'$ whose order corresponds to the order in which there
  are used by $\plan'$. Clearly, $\plan$ is a plan for $\iplan$. 
  It remains to show that $l\leq k$ for which we need the following
  claim.
  \begin{claim}\label{APPENDIX-clm:many-or-no-actions}
    If $\plan'$ uses some action $a \in \acts$ then $\plan'$
    contains at least $k+2$ actions from $a_1(a),\dotsc,a_{k+3}(a)$.
  \end{claim}
  Let $i$ be the largest integer with $1 \leq i \leq
  k+3$ such that $a_i(a)$ occurs in $\plan'$. We first show by
  induction on $i$ that
  $\plan'$ contains all actions in $\SB a_j(a) \SM 1 \leq j \leq i
  \SE$. Clearly, if $i=1$ there is nothing to show, so assume that
  $i>1$. The induction step follows from the fact that the action $a_i(a)$ has a
  bad effect on the variable $v_{i-1}(a)$ and the action
  $a_{i-1}(a)$ is the only action of $\iplan'$ that has a good effect on
  $v_{i-1}(a)$ and hence $\plan'$ has to contain the action
  $a_{i-1}(a)$. It remains to show that $i \geq k+2$.
  Suppose for a contradiction that $i < k+2$ and consequently the
  action $a_{i+1}(a)$ is not contained in $\plan'$. Because the action
  $a_{i+1}(a)$ is the only action of $\iplan'$
  that has a bad effect on the variable $v_i(a)$ it follows that the
  variable $v_i(a)$ remains in the goal state over the whole execution
  of the plan $\plan'$. But then $\plan'$ without the action $a_i(v)$
  would still be a plan for $\iplan'$ contradicting our assumption
  that $\iplan'$ is minimal with respect to sub sequences. 

  It follows from Claim~\ref{APPENDIX-clm:many-or-no-actions} that $\plan'$ uses at most
  $\frac{l'}{k+2}\leq \frac{k'}{k+2}=\frac{k(k+3)+1}{k+2}<k+1$ actions
  from $\acts$. Hence, $l \leq k$ proving the lemma.
\end{proof}
 
\medskip

We are now ready to show that $(0,2)$-\BPE{} is fixed-parameter tractable.
\begin{theorem}\label{the:02fpt}
  $(0,2)$-\BPE{} is fixed-parameter tractable.
\end{theorem}
\begin{proof}
  We show fixed-parameter tractability of $(0,2)$-\BPE{} by reducing it
  to the following fixed-parameter tractable problem. The problem has
  originally been shown to be fixed-parameter tractable for undirected
  graphs~\cite{DreyfusWagner71}. Later Guo, Niedermeier, and Suchy mentioned that
  this result can be directly transferred to the directed case~\cite{GuoNiedermeierSuchy11}.
  \smallskip
  \begin{quote}
    \noindent
    \textsc{Directed Steiner Tree}\\
    \noindent  
    \emph{Instance:} A set of nodes $N$, 
    a weight function $w\ :\  N \times N \rightarrow (\mathbb{N} \cup
    \{\infty\})$, 
    a root node $s \in N$,
    a set $T \subseteq N$ of terminals , and a weight bound $p$.\\
    \noindent  
    \emph{Parameter:} $p_M=\frac{p}{\min\SB w(u,v) \SM u,v \in N\SE}$.\\  
    \noindent
    \emph{Question:} Is there a set of arcs $E \subseteq N \times N$ of
    weight $w(E) \leq p$ (where $w(E)=\sum_{e \in E}w(e)$) such that in
    the digraph $D=\tuple{N,E}$ for every $t \in T$ there is a directed path
    from $s$ to $t$? We will call the digraph $D$ a \emph{directed
      Steiner Tree (DST)} of weight $w(E)$.
  \end{quote}
  \smallskip Let $\insti=\tuple{\iplan,k}$ where
  $\iplan=\tuple{\vars,\dom,\acts,\init,\goal}$ be an instance of
  $(0,2)$-\BPE{}. Because of Lemma~\ref{lem:nodoublegoodactions} we can
  assume that $\acts$ contains no good actions with two effects.  We
  construct an instance $\insti'=\tuple{N,w,s,T,p}$ of \textsc{Directed
    Steiner Tree} where $p_M=k$ such that 
  $\iplan$ has a plan of length at most $k$ if and only if
  $\insti'$ has a directed Steiner tree of weight at most $p$.
  Because $p_M=k$ this shows that $(0,2)$-\BPE{} is fixed-parameter tractable.

  We are now ready to define
  the instance $\insti'$. 
  The node set $N$ 
  consists of the root vertex $s$
  and one node for every variable in $\vars$. The weight function
  $w$ is $\infty$ for all but the following arcs:
  \begin{itemize}
  \item[(i)] For every good action $a \in \acts$ the arc from $s$ to the
    unique variable $v \in \vars$ that is affected by $a$ gets weight~$1$.
  \item[(ii)] For every mixed action $a \in \acts$ with a good effect on
    some variable $v_g \in \vars$ and a bad effect on some variable
    $v_b \in \vars$, the arc from $v_b$ to $v_g$ gets weight~$1$.
  \end{itemize}
  We identify the root $s$ from the instance $\insti$ 
  with the node $s$, we let $T$ be the set $\Delta(\vars)$, and
  $p_M=p=k$.
  \begin{claim}\label{claim:dst}
    $\iplan$ has a plan of length at most $k$ if and only if $\insti'$
    has a DST of weight at most $p_M=p=k$.
  \end{claim}
  Suppose $\iplan$ has a plan $\plan=\seq{a_1,\dots,a_l}$ with $l
  \leq k$. Without losing generality we can assume that $\plan$ contains 
  no bad
  actions.
  The arc set $E$ that corresponds to $\plan$ consists of
  the following arcs:
  \begin{itemize}
  \item[(i)] For every good action $a \in \plan$ that has its unique good
    effect on a variable $v \in \vars$, the set $E$ contains the arc
    from $s$ to $v$.
  \item[(ii)] For every mixed action $a \in \plan$ with a good effect on some
    variable $v_g$ and a bad effect on some variable $v_b$, the set
    $E$ contains an arc from $v_b$ to $v_g$.
  \end{itemize}
Intuitively, for case (ii), note that $a$ has a bad effect on $v_b$,
i.e. it sets $v_b$ to a different value than its goal value, 
so $a$ must be followed by some sequence $a_1,\dots,a_n$ of actions
where only the last one is good and the others are mixed.
This will provide a path in the DST from the root to $v_b$.

  More formally, it follows that the weight of $E$ equals the number of
  actions in $\plan$ and hence is at most $p=k$ as
  required. It remains to show that the digraph $D=(V,E)$ is a
  DST, i.e., $D$ contains a directed path from the vertex $s$ to
  every vertex in $T$. 
  Suppose to the contrary that there is a terminal
  $t \in T$ that is not reachable from $s$ in~$D$. Furthermore, let $R
  \subseteq E$ be the set of all arcs in $E$ such that $D$ contains a
  directed path from the tail of every arc in $R$ to $t$. It follows
  that no arc in $R$ is incident to $s$. Hence, $R$ only consists of
  arcs that correspond to mixed actions in $\plan$. If $R=\emptyset$
  then the plan $\plan$ does not contain an action that affects the
  variable $t$. But this contradicts our assumption that $\plan$ is a
  plan (because $t \in \Delta(\vars)$). Hence, $R \neq \emptyset$. 
  Let $a$ be
  the mixed action corresponding to the arc in $R$ that occurs last in
  $\plan$ (among all mixed actions that correspond to an arc in
  $R$). Furthermore, let $v \in \vars$ be the variable that is badly
  affected by $a$. Then $\plan$ cannot be a plan because after the
  occurrence of $a$ in $\plan$ there is no action in $\plan$ that
  affects $v$ and hence $v$ cannot be in the goal state after $\plan$
  is executed.

  To see the reverse direction, let $E \subseteq N \times N$ be a
  solution of $\insti$ and let $D=(N,E)$ be the {DST}.  
  Without losing generality we can
  assume that $D$ is a directed acyclic tree rooted in $s$ (this follows
  from the minimality of $D$). We obtain a plan $\plan$ of length at
  most $p$ for $\iplan$ by traversing the DST $D$ in a bottom-up
  manner. More formally, let $d$ be the maximum distance from $s$ to any
  node in $T$, and for every $1 \leq i < d$ let $A(i)$ be the set of
  actions in $\acts$ that correspond to arcs in $E$ whose tail is at
  distance $i$ from the node $s$. Then $\plan=\seq{A(d-1),\dots,A(1)}$
  (for every $1 \leq i \leq d-1$ the actions contained in $A(d-1)$ can
  be executed in an arbitrary order) is a plan of length at most $k=p$
  for $\iplan$.
\end{proof}

\section{Kernel Lower Bounds}
\noindent
\sloppypar In the previous sections
we have classified the parameterized complexity of \BPE{}. 
It turned out that the problems fall into four
categories  (see Figure~\ref{fig:pubs-lattice}):
\begin{itemize}
\item[(i)] polynomial-time solvable, 
\item[(ii)] NP-hard but fixed-parameter tractable, 
\item[(iii)] $\W{1}$\hy complete, and
\item[(iv)] $\W{2}$\hy complete.
\end{itemize}
The aim
of this section is to further refine this classification with respect to
kernelization. The problems in category~(i) trivially admit a kernel of
constant size, whereas the problems in categories~(iii) and (iv) do not admit a
bi-kernel at all (polynomial or not), unless $\W{1}=\FPT$ or $\W{2}=\FPT$,
respectively. Hence it remains to consider the problems in
category~(ii), each of them could either admit a polynomial bi-kernel or
not. We show that none of them does.

\subsection{Kernel Lower Bounds for PUBS Restrictions}
 
According to our classification so far, the only problems in
category~(ii) with respect to the PUBS-restrictions are the problems
$R$-\BPE{}, for $R \subseteq \{P,U,B,S\}$ such that $P\in R$ and
$\{P,U,S\} \not\subseteq R$.
\begin{theorem}\label{thm:nopolykernel-p}
  None of the problems $R$-\BPE{} for $R \subseteq \{P,U,B,S\}$ such
  that $P\in R$ and $\{P,U,S\} \not\subseteq R$ (i.e., the problems in
  category~(ii)) admits a polynomial bi-kernel unless $\NOPOLYKERNEL$.
\end{theorem}
The remainder of this section is devoted to
establish the above theorem.  The relationships between
the problems as indicated in Figure~\ref{fig:pubs-lattice} greatly
simplify the proof.  Instead of considering all six problems
separately, we can focus on the two most restricted problems
$\{P,U,B\}$-\BPE{} and $\{P,B,S\}$-\BPE{}.  If any other problem in
category~(ii) would have a polynomial bi-kernel, then at least one of these
two problems would have one. This follows by
Proposition~\ref{pro:poly-par-reduction-bi} and the following facts:
\begin{enumerate}
\item The unparameterized versions of all the problems in category~(ii)
  are \NP-hard. This holds since  the corresponding classical
  problems are strongly \NP-hard, i.e., the problems remain \NP-hard
  when $k$ is encoded in unary (as shown by B\"{a}ckstr\"{o}m and Nebel
  \cite{BackstromNebel95});
\item If $R_1\subseteq R_2$ then the identity function gives a
  polynomial parameter reduction from $R_2$-\BPE{} to $R_1$-\BPE{}.
\end{enumerate}
Furthermore, the following result of B\"{a}ckstr\"{o}m and Nebel
even provides a polynomial
parameter reduction from $\{P,U,B\}$-\BPE{} to
$\{P,B,S\}$-\BPE{}. Consequently, $\{P,U,B\}$-\BPE{} remains the only
problem for which we need to establish a super-polynomial bi-kernel lower
bound.
\begin{proposition}[\mbox{B\"{a}ckstr\"{o}m and Nebel 
\cite[Theorem 4.16]{BackstromNebel95}}] \label{pro:from-pb-to-pbs}
  Let $\insti=\tuple{\iplan,k}$ be an instance of
  $\{P,U,B\}$-\BPE{}. Then $\insti$ can be transformed in polynomial
  time into an equivalent instance $\insti'=\tuple{\iplan',k'}$ of
  $\{P,B,S\}$-\BPE{} such that $k=k'$.
\end{proposition}
Hence, in order to complete the proof of
Theorem~\ref{thm:nopolykernel-p} it only remains to establish the next
lemma.
\begin{lemma}\label{lem:nopolykernel-pub}
  $\{P,U,B\}$-\BPE{}
 has no polynomial bi-kernel unless $\NOPOLYKERNEL$.
\end{lemma}
\begin{proof}
  Because of Proposition~\ref{pro:strong-or-comp-no-poly-kernel}, it suffices
  to devise a strong OR-composition algorithm for $\{P,U,B\}$-\BPE{}.  Suppose
  we are given $t$ instances
  $\insti_1=\tuple{\iplan_1,k_1},\dots,\insti_t=\tuple{\iplan_t,k_t}$ of
  $\{P,U,B\}$-\BPE{} where
  $\iplan_i=\tuple{\vars_i,\dom_i,\acts_i,\init_i,\goal_i}$ for every $1
  \leq i \leq t$. Let $k=\max_{1\leq i \leq t}k_i$. 
  According to Theorem~\ref{the:Pfpt},
  $\{P,U,B\}$-\BPE{} can be solved in time
  $O^*(S(k))$ (where $S(k)=k^k$ 
  and the $O^*$ notation suppresses
  polynomial factors).
  It follows that $\{P,U,B\}$-\BPE{} can be solved in polynomial time
  with respect to $\sum_{1 \leq i \leq t}|\insti_i|+k$ if
  $t>S(k)$. Hence, if $t>S(k)$ this gives us a strong OR-composition algorithm
  as follows. We first run the algorithm for $\{P,U,B\}$-\BPE{} on each
  of the $t$ instances. 
  If there is some $i$, where $1 \leq i \leq t$, such that
  $\iplan_i$ has a plan of length at most $k_i$,
  then arbitrarily choose such an $i$ and output $\insti_i$.
  Otherwise, arbitrarily output any of the instances $\insti_1,\dots,\insti_t$.
This shows that $\{P,U,B\}$-\BPE{}
  has a strong OR-composition algorithm for the case where $t>S(k)$. Hence, in
  the following we can assume that $t \leq S(k)$.

  Given $\insti_1,\dots,\insti_t$ we will construct an instance
  $\insti=\tuple{\iplan,k'}$ of $\{P,U,B\}$-\BPE{} as follows. For the
  construction of $\insti$ we need the following auxiliary gadget, which will be
  used to calculate the logical ``OR'' of two binary variables. The
  construction of the gadget uses ideas from B{\"a}ckstr{\"o}m and
  Nebel~\cite[Theorem 4.15]{BackstromNebel95}. Assume
  that $v_1$ and $v_2$ are two binary variables. The
  gadget $\textup{OR}_2(v_1,v_2,o)$ consists of the five binary
  variables $o_1$, $o_2$, $o$, $i_1$, and $i_2$. Furthermore,
  $\textup{OR}_2(v_1,v_2,o)$ contains the following actions:
  \begin{itemize}
  \item the action $a_o$ with
    $\proj{\pre(a_o)}{o_1}=\proj{\pre(a_o)}{o_2}=1$ and $\proj{\eff(a_o)}{o}=1$;
  \item the action $a_{o_1}$ with
    $\proj{\pre(a_{o_1})}{i_1}=1$, $\proj{\pre(a_{o_1})}{i_2}=0$ and
    $\proj{\eff(a_{o_1})}{o_1}=1$;
  \item the action $a_{o_2}$ with
    $\proj{\pre(a_{o_2})}{i_1}=0$, $\proj{\pre(a_{o_2})}{i_2}=1$ and $\proj{\eff(a_{o_2})}{o_2}=1$;
  \item the action $a_{i_1}$ with
    $\proj{\eff(a_{i_1})}{i_1}=1$;
  \item the action $a_{i_2}$ with
    $\proj{\eff(a_{i_2})}{i_2}=1$;
  \item the action $a_{v_1}$ with $\proj{\pre(a_{v_1})}{v_1}=1$ and
    $\proj{\eff(a_{v_1})}{i_1}=0$;
  \item the action $a_{v_2}$ with $\proj{\pre(a_{v_2})}{v_2}=1$ and
    $\proj{\eff(a_{v_2})}{i_2}=0$;
  \end{itemize}
  We now show that $\textup{OR}_2(v_1,v_2,o)$ can indeed be used to
  compute the logical ``OR'' of the variables $v_1$ and $v_2$. We need
  to show the following claim.
  \begin{claim}\label{clm:pub-nokernel-2or}
    Let $\iplan(\textup{OR}_2(v_1,v_2,o))$ be a $\{P,U,B\}$-\BPE{} instance
    that consists of the two binary variables $v_1$ and $v_2$, and the
    variables and actions of the gadget
    $\textup{OR}_2(v_1,v_2,o)$. Furthermore, let the initial state of 
    $\iplan(\textup{OR}_2(v_1,v_2,o))$ be any initial state that sets
    all variables of the gadget $\textup{OR}_2(v_1,v_2,o)$ to $0$ but
    assigns the variables $v_1$ and $v_2$ arbitrarily, and let the
    goal state of $\iplan(\textup{OR}_2(v_1,v_2,o))$ 
    be defined by $\proj{\goal}{o}=1$. Then
    $\iplan(\textup{OR}_2(v_1,v_2,o))$ has a plan if and only if its
    initial state sets at least one of the variables $v_1$ or $v_2$ to
    $1$. Furthermore, if there is such a plan then its length is $6$.
  \end{claim}
  Suppose that there is a plan $\plan$ for $\iplan(\textup{OR}_2(v_1,v_2,o))$
  and assume for a contradiction that both variables $v_1$ and $v_2$
  are initially set to $0$. It is easy to see that the value of $v_1$
  and $v_2$ can not change during the whole duration of the plan and that 
  $\plan$ has to
  contain the actions $a_{o_1}$ and $a_{o_2}$. Without losing
  generality we can assume that
  $\plan$ contains $a_{o_1}$ before it contains $a_{o_2}$. Because of
  the preconditions of the actions $a_{o_1}$ and $a_{o_2}$, 
  the variable $i_1$ must have value
  $1$ before $a_{o_1}$ occurs in $\plan$ and it must have value $0$
  before the action $a_{o_2}$ occurs in $\plan$. Hence, $\plan$ must
  contain an action that sets the variable
  $i_1$ to $0$. However, this can not be the case, since the only
  action setting $i_1$ to $0$ is the action $a_{v_1}$ which can not occur
  in $\plan$ because the variable $v_1$ is $0$ for the whole
  duration of $\plan$.

  To see the reverse direction suppose that one of the variables $v_1$
  or $v_2$ is initially set to $1$. If $v_1$ is initially set to one
  then $\seq{a_{i_1},a_{o_1},a_{v_1},a_{i_2},a_{o_2},a_o}$ is a plan
  of length $6$ for $\iplan(\textup{OR}_2(v_1,v_2,o))$. On the other
  hand, if $v_2$ is initially set to one
  then $\seq{a_{i_2},a_{o_2},a_{v_2},a_{i_1},a_{o_1},a_o}$ is a plan
  of length $6$ for $\iplan(\textup{OR}_2(v_1,v_2,o))$.
  Hence the claim is true.
  It should be noted that this is a use-once gadget;
  when it has computed the disjunction of $v_1$ and $v_2$ it may
  not be possible to reset it to do this computation again.
  This is sufficient for our purpose, however.

  We continue by showing how
  we can use the gadget $\textup{OR}_2(v_1,v_2,o)$ to construct
  a gadget $\textup{OR}(v_1,\dots,v_r,o)$ such that there
  is a sequence of actions of $\textup{OR}(v_1,\dots,v_r,o)$ that
  sets the variable $o$ to $1$ if and only if at least one of the external
  variables $v_1,\dots,v_r$ are initially set to $1$. Furthermore, if
  there is such a sequence of actions then its length is at most $6
  \lceil \log r \rceil$. Let $T$ be a
  rooted binary tree with root $s$ that has $r$ leaves $l_1,\dots,l_r$ and is of
  smallest possible height. For every node $t \in V(T)$ we make a copy
  of our binary OR-gadget such that the copy of a leave node $l_i$ is
  the gadget $\textup{OR}_2(v_{2i-1},v_{2i},o_{l_i})$ and the
  copy of an inner node $t \in V(T)$ with children $t_1$ and $t_2$ is the gadget
  $\textup{OR}_2(o_{t_1},o_{t_2},o_{t})$ (clearly this needs to be adapted
  if $r$ is odd or an inner node has only one child). For the root
  node with children $t_1$ and $t_2$ the gadget becomes
  $\textup{OR}_2(o_{t_1},o_{t_2},o)$.
  This completes
  the construction of the gadget
  $\textup{OR}(v_1,\dots,v_r,o)$. Using
  Claim~\ref{clm:pub-nokernel-2or} it is easy to verify that the
  gadget $\textup{OR}(v_1,\dots,v_r,o)$ can indeed be used to compute
  the logical ``OR'' or the variables $v_1,\dotsc,v_r$.

  We are now ready to construct the instance $\insti$. $\insti$
  contains all the variables and actions from every instance
  $\insti_1,\dots,\insti_t$ and of the gadget
  $\textup{OR}(v_1,\dots,v_t,o)$. Furthermore, for every $1 \leq i
  \leq t$ and $k_i \leq j \leq k$, the instance $\insti$ contains the
  binary variables $p_j^i$ and the actions $a^i_j$ such that:
  \begin{itemize}
  \item $\pre(a^i_{k_i})=\goal_i$ and
    $\proj{\eff(a^i_{k_i})}{p_{k_i}^i}=1$,
  \item $\proj{\pre(a^i_{k_i+l})}{p_{k_i+l-1}^i}=1$ and
    $\proj{\eff(a^i_{k_i+l})}{p_{k_i+l}^i}=1$, for every $1 \leq l
    \leq k-k_i$.
  \end{itemize}
  Note that the actions $a_j^i$ and the variables $p_j^i$ are used to
  ``pad'' the different parameter values of the instances
  $\insti_1,\dotsc,\insti_t$ to the value $k$.

  Additionally, $\insti$ contains the binary variables 
  $v_1,\dots,v_t$ and the actions $a_1,\dots,a_t$ with
  $\proj{\pre(a_i)}{p_k^i}=1$ and $\proj{\eff(a_i)}{v_i}=1$.
  Furthermore, the initial state
  $\init$ of $\insti$ is defined as $\proj{\init}{v}=\proj{\init_i}{v}$
  if $v$ is a variable of $\insti_i$ and $\proj{\init}{v}=0$,
  otherwise. The goal state of $\insti$ is defined by
  $\proj{\goal}{o}=1$ and we set $k'=k+1+6\lceil \log t
  \rceil$. Clearly, $\insti$ can be constructed from
  $\insti_1,\dots,\insti_t$ in polynomial time
  and $\iplan$ has a plan of length at most $k$
  if and only if there is some $i$, where $1 \leq i \leq t$,
  such that $\iplan_i$ has a plan of length at most $k_i$.
  Furthermore, because 
  $k' = k+1+6\lceil \log t \rceil \leq k+1+6\lceil \log S(k) \rceil 
  = k+1+6\lceil \log k^k\rceil  = k+1+6\lceil k\log k  \rceil$, the parameter $k'$
  is polynomially bounded by the parameter $k$. This concludes the proof
  of the lemma.
\end{proof}

\subsection{Kernel Lower Bounds for $(0,2)$-\BPE{}}

According to our classification so far, the only problem in
category~(ii) with respect to restrictions on the number of
preconditions and effects is $(0,2)$-\BPE{}. The next theorem suggests
that $(0,2)$-\BPE{} has no polynomial bi-kernel.
\begin{theorem} \label{02bpenokernel}
  $(0,2)$-\BPE{} has no polynomial bi-kernel unless
  $\NOPOLYKERNEL$.
\end{theorem}
\begin{proof}\sloppypar
  It is apparent from the \NP-completeness proof for $(0,2)$-\BPE{}
  \cite[Theorem~4.6]{Bylander94} that the problem is even strongly
  \NP-complete, i.e., the unparameterized version of it is
  \NP-complete.
  According to Proposition~\ref{pro:or-comp-no-poly-kernel} it is
  thus sufficient to devise an OR-composition algorithm for 
  $(0,2)$-\BPE{} to prove the theorem.
  Suppose we are given
  $t$ instances
  $\insti_1=\tuple{\iplan_1,k},\dots,\insti_t=\tuple{\iplan_t,k}$ of
  $(0,2)$-\BPE{} where
  $\iplan_i=\tuple{\vars_i,\dom_i,\acts_i,\init_i,\goal_i}$ for every
  $1 \leq i \leq t$. We will now show how we can construct the
  required instance $\insti=\tuple{\iplan,k''}$ of $(0,2)$-\BPE{} via
  an OR-composition algorithm. 
  Without losing generality, we assume that $\iplan_1,\dots,\iplan_t$
  have disjoint sets of variables and disjoint sets of actions.
  As a first step we compute the new instances
  $\insti_1'=\tuple{\iplan_1',k'},\dots,\insti_t'=\tuple{\iplan_t',k'}$
  from 
  $\insti_1=\tuple{\iplan_1,k},\dots,\insti_t=\tuple{\iplan_t,k}$
  according to Lemma~\ref{lem:nodoublegoodactions}. Then
  $\vars$ consists of the following variables:
  \begin{itemize}
  \item[(i)] the variables $\bigcup_{1\leq i \leq t}\vars_i'$;
  \item[(ii)] binary variables $b_1,\dotsc,b_{k'}$;
  \item[(iii)] for every $1 \leq i \leq t$ and $1 \leq j < 2k'$ a binary
    variable $p_{i,j}$;
  \item[(iv)] A binary variable $r$.
  \end{itemize}
  $\acts$ contains the action $a_r$ with $\proj{\eff(a_r)}{r}=0$ and
  the following additional actions for every $1 \leq i \leq t$:
  \begin{itemize}
  \item[(i)] The actions $\acts_i' \setminus a_g^i$, where $a_g^i$ is the
    copy of the action $a_g$ for the instance $\insti_i'$ (recall the
    construction of $\insti_i'$ given in
    Lemma~\ref{lem:nodoublegoodactions});
  \item[(ii)] An action $a_i(r)$ with $\proj{\eff(a_i(r))}{r}=1$ and
    $\proj{\eff(a_i(r))}{p_{i,1}}=0$;
  \item[(iii)] For every $1 \leq j < 2k'-1$ an action
    $a_{i,j}$ with $\proj{\eff(a_{i,j})}{p_{i,j}}=1$ and
    $\proj{\eff(a_{i,j})}{p_{i,j+1}}=0$;
  \item[(iv)] An action $a_i(g)$ with
    $\proj{\eff(a_i(g))}{p_{i,2k'-1}}=1$ and
    $\proj{\eff(a_i(g))}{g^i}=0$ where $g^i$ is the
    copy of the variable $g$ for the instance $\insti_i'$ (recall the
    construction of $\insti_i'$ given in
    Lemma~\ref{lem:nodoublegoodactions});
  \item[(v)] Let $v_1,\dotsc,v_r$ for $r \leq k'$ be an
    arbitrary ordering of the variables in $\Delta(\vars_i)$
    (recall the
    definition of $\Delta(\vars_i)$ from Section~\ref{sec:res-fpt}). 
    Then for every $1 \leq j \leq r$ we introduce an action
    $a_i(b_j)$ with $\proj{\eff(a_i(b_j))}{v_j}=\proj{\init_i'}{v_j}$ and
    $\proj{\eff(a_i(b_j))}{b_j}=0$. Furthermore, for 
    every $r < j \leq k'$ we introduce an action $a_i(b_j)$ with
    $\proj{\eff(a_i(b_j))}{v_r}=\proj{\init_i'}{v_r}$ and
    $\proj{\eff(a_i(b_j))}{b_j}=0$.
  \end{itemize}
  We set $\dom=\bigcup_{1 \leq i \leq t}\dom_i' \cup \{0,1\}$,
  $\proj{\init}{v}=\proj{\goal_i'}{v}$
  for every $v \in \vars_i'$ and $1 \leq i \leq t$, 
  $\proj{\init}{v}=0$ for every $v \in \vars \setminus ((\bigcup_{1
    \leq i \leq t}\vars_i') \cup \{b_1,\dotsc,b_{k'}\})$, 
  $\proj{\init}{v}=1$ for every $v \in \{b_1,\dotsc,b_{k'}\}$,
  $\proj{\goal}{v}=\proj{\goal_i'}{v}$
  for every $v \in \vars_i'$ and $1 \leq i \leq t$, 
  $\proj{\goal}{v}=0$ for every $v \in \vars \setminus (\bigcup_{1
    \leq i \leq t}\vars_i')$, and $k''=4k'+1$.

  We note that all the subinstances corresponding to
  $\iplan'_1,\dots,\iplan'_t$ already have their goals satisfied
  in the initial state $\init$.
  However, since the variables $b_1,\dots,b_k$ have the wrong
  value in $\init$ it is necessary to include actions of 
  type $a_i(b_j)$ in the plan.
  This ``destroys the goal'' for at least one subinstance so
  the plan must include a subplan to solve also this subinstance.

  Clearly, $\insti$ can be constructed from
  $\insti_1,\dots,\insti_t$ in polynomial time with respect to
  $\sum_{1\leq i \leq t}|\insti_i|+k$
  and the parameter $k''=4k'+1=4(k(k+3)+1)+1$
  is polynomially bounded by the parameter $k$. 
  By showing the following claim we conclude the proof of
    the theorem.
  \begin{claim}\label{clm:or}
    $\iplan$ has a plan of length at most $k$ if and
    only if at least one of $\iplan_1,\dots,\iplan_t$ 
    has a plan of length at most $k$.
  \end{claim}

  Suppose that there is an $1 \leq i \leq t$ such that $\iplan_i$
  has a plan of length at most $k$. It follows from
  Lemma~\ref{lem:nodoublegoodactions} 
  that $\iplan_i'$ has a plan $\plan'$ of length at most $k'$. Then it
  is straightforward to check that 
  $\plan=\seq{a_i(b_1), \dotsc,a_i(b_{k'})} \concat \plan' 
  \concat \seq{a_i(g),a_{i,2k'-2},\dotsc,a_{i,1},a_i(r),a_r}$ 
  is a plan of length at most $4k'+1$ for $\iplan$.

  For the reverse direction let $\plan$ be a plan of length at most
  $k''$ for $\iplan$. Without losing generality we can assume that 
  for every $1\leq i
  \leq t$ the set $\Delta(\vars_i)$ is not empty and hence every
  plan for $\iplan_i'$ has to contain at least one action $a \in
  \acts_i'$ that corresponds to a good action in $\iplan_i$. 
  Because $\proj{\eff(a)}{g^i} \neq \init_i'[g]=\goal_i'[g]$ for every such
  good action $a$ (recall the construction of $\insti_i'$ according to
  Lemma~\ref{lem:nodoublegoodactions}) it follows that there is an $1
  \leq i \leq t$ such that $\plan$ contains all the $2k'+1$ actions
  $a_i(g),a_{i,2k'-2},\dotsc,a_{i,1},a_i(r),a_r$. Furthermore, because
  $k''<2(2k'+1)$ there can be at most one such $i$ and hence
  $\plan \cap \bigcup_{1 \leq j \leq t}\acts_j' \subseteq
  \acts_i'$. Because $\Delta(\vars)=\{b_1,\dotsc,b_{k'}\}$ the plan $\plan$
  also has to contain the actions
  $a_i(b_1),\dotsc,a_i(b_{k'})$. Because of the effects (on the
  variables in $\Delta(V_i)$) of these
  actions it follows that $\plan$ has to contain a plan $\plan_i'$ of
  length at most $4k'+1-(2k'+1)-k'=k'$ for $\iplan_i'$. It now follows
  from Lemma~\ref{lem:nodoublegoodactions} that $\iplan_i$ has a plan
  of length at most~$k$.
\end{proof}

\section{Summary of Results}\label{sec:class}
\noindent
We have obtained a full classification of the parameterized complexity
of planning 
with respect to the length of the solution plan, under 
all combinations of the syntactical P, U, B, and S restrictions 
previously considered by B\"{a}ckstr\"{o}m and 
Nebel~\cite{BackstromNebel95}.
The complexity results for the various combinations of
restrictions P, U, B and S are displayed in
Figure~\ref{fig:pubs-lattice}.
Solid lines denote separation results by
B\"{a}ckstr\"{o}m and Nebel
\cite{BackstromNebel95},
using standard complexity analysis,
while dashed lines denote separation results 
from our parameterized analysis.
The $\W{2}$-completeness results follow from
Theorems~\ref{arb-hard}~and~\ref{in-W2},
the $\W{1}$-completeness results follow from
Theorems~\ref{unary-hard}~and~\ref{in-W1},
and the \FPT\ results follow from
Theorems~\ref{the:Pfpt}~and~\ref{the:02fpt}.
Finally, Theorem~\ref{thm:nopolykernel-p} shows that none of the
variants of \BPE{}, which are \NP-hard and in \FPT{}, 
admit a polynomial bi-kernel.

Bylander~\cite{Bylander94} studied the complexity of \strips\ under
varying numbers of preconditions and effects, which is natural to view
as a relaxation of restriction U in \sasplus.  
We provide a full classification of the parameterized complexity of
planning under Bylander's restrictions.  
Table~\ref{table:bylander} shows such results 
(for arbitrary domain sizes $\geq 2$)
under both parameterized and classical analysis.
The parameterized results in Table~\ref{table:bylander} 
are derived as follows.
For actions with an arbitrary number of effects, 
the results follow from Theorems~\ref{arb-hard} and~\ref{in-W2}.
For actions with at most one effect, 
we have two cases: 
With no preconditions
the problem is trivially in \poly. Otherwise, the results
follow from Theorems~\ref{unary-hard} and~\ref{in-W1}.
The case where the number of effects is
bounded by some constant $m_e > 1$ can be reduced in polynomial time
to the case with only one effect using a reduction by 
B{\"a}ckstr{\"o}m~\cite[proof of Theorem~6.7]{Backstrom92}.
Since this reduction is a parameterized reduction
we have membership in \W{1} by Theorem~\ref{in-W1}. 
When $m_p \geq 1$, then we also have \W{1}-hardness by
Theorem~\ref{unary-hard}.
For the final case ($m_p=0$), we obtain $\W{1}$-hardness from
Theorem~\ref{the:03-BPE-hard} and containment in \W{1}
from Theorem~\ref{in-W1} if the number of effects $m_e$ is at least $3$.
The case where also $m_e=2$ is fixed-parameter tractable
according to Theorem~\ref{the:02fpt},
but Theorem~\ref{02bpenokernel} excludes that it admits 
a polynomial bi-kernel.

Since \W{1} and \W{2} are not directly comparable
to the standard complexity classes we get interesting
separations from combining the two methods.
For instance, we can single out restriction U as making planning
easier than in the general case, which is not possible under standard
analysis.
Since planning remains as hard as in the general case under
restrictions B and S also for parameterized analysis, 
it seems that U is a more interesting
and important restriction than the other two.
Furthermore, the results in Table~\ref{table:bylander} suggest
that also the restriction to a fixed number of effects larger
than one is an interesting case.
Even more interesting is that planning is in \FPT\ under restriction P,
making it easier than the combination restriction US, 
while it seems to be rather the other way around for 
standard analysis where restriction P is only known
to be hard for \NP.

We have also provided a full classification of bi-kernel sizes for all
the fixed-parameter tractable fragments.  It turns out that none of
the nontrivial problems (where the unparameterized version is
$\NP$-hard) admits a polynomial bi-kernel unless the Polynomial-time
Hierarchy collapses. This implies an interesting dichotomy concerning
the bi-kernel size: we only have constant-size and super-polynomial
bi-kernels, and polynomially bounded bi-kernels that are not of
constant size are absent.  In order to establish these results, we had
to adapt standard tools for kernel lower bounds to parameterized
problems whose unparameterized versions are not (or not known to be)
in $\NP$. We think that our notion of a strong OR-composition and the
corresponding Proposition~\ref{pro:strong-or-comp-no-poly-kernel}
could be useful for showing kernel lower bounds for other
parameterized problems whose unparameterized versions are
outside~$\NP$.

\section{Discussion}
\noindent
This work opens up several new research directions. We briefly
discuss some of them below.

The use of parameterized analysis in planning is by no means
restricted to using plan length as parameter.  For instance, very
recently Kronegger et~al.~\cite{KroneggerPfandlerPichler13} obtained
parameterized results for several different parameters and
combinations of them: one should note that the parameter need not be a
single value,it can be a combination of two or more `basic'
parameters.  A second example is considered by Downey et
al.~\cite{DowneyFellowsStege99}.  They show that \strips\ planning can
be recast as the \textsc{Signed Digraph Pebbling} problem which is
modeled as a special type of graph.  They analyze the parameterized
complexity of this problem considering also the treewidth of the graph
as a parameter.  
A final example is the recent paper by de Haan et
al.~\cite{deHahnRoubickovaSzeider13} who study the parameterized
complexity of plan reuse, where the task is to modify an existing plan
to obtain the solution for new planning instance by making a small
modification.

Our observation that restriction U makes planning easier
under parameterized analysis 
is interesting in the context of the literature on planning.
Although this case remains \PSPACE-complete under classical
complexity analysis, it has been repeatedly stressed in the
literature that unary actions are interesting for reasons of
efficiency.
For instance, Williams and Nayak~\cite{WilliamsNayak97}
considered planning for spacecrafts and found that unary
actions were often sufficient to model real-world problems
in this domain.
They noted that one consequence of having only unary
actions is that the causal graph for a planning instance
must be acyclic, a property which has often been exploited
in the literature, both for theoretical results on
planning complexity for various strucutures of the causal
graph \cite{BrafmanDomshlak03,GimenezJonsson08,JonssonBackstrom98b}
as well as for practical planning~\cite{Helmert04,Helmert06b}.
Of particular interest is a result on causal graphs in general
by Chen and Gim\'{e}nez \cite{ChenGimenez10}.
It is a classical complexity result that is proven under an assumption
from  parameterized complexity.

There are also close ties between model checking and planning
and this connection deserves further study.
For instance, model-checking traces can be viewed as plans
and vice versa~\cite{EdelkampLeueVisser07},
and methods and results have been transferred between 
the two areas in both directions 
\cite{Edelkamp03,EdelkampKellershoffSulewski10,WehrleHelmert09}.
Our reductions from planning to model-checking suggest that 
the problems are related also on a more fundamental level than
just straightforward syntactical translations.

The major motivation for our research is the need for alternative
and complementary methods in complexity analysis of planning.
However, planning is also an interesting problem per se.
It is a very powerful modelling language since it is \PSPACE-complete
in the general case, while it is also often simple to model other
problems as planning problems. 
For instance, it would be interesting to identify various  restrictions 
that make planning \NP-complete but still allow for straightforward
modelling of many common \NP-complete problems, 
and analogously for other classes than \NP.

Like most other results on complexity analysis of planning
in the literature, our results apply to various restrictions
of the actual planning language.
However, the commonly studied language restrictions usually do not
match the restrictions implied by applications.
A complementary approach is thus to study the complexity of
common benchmark problems for planning, 
e.g. the blocks world~\cite{GuptaNau92} and
the problems used in the international planning 
competitions~\cite{Helmert03,Helmert06}.
Hence, it would be interesting to apply parameterized complexity 
analysis to these problems to see if it could help to explain
the empirical results on which problems are hard and easy in
practice.
Finding the right parameter(s) would, of course, be crucial
for achieving relevant results on this.


\begin{thebibliography}{10}

\bibitem{AlonGutinKimSzeiderYeo11}
Noga Alon, Gregory Gutin, Eun~Jung Kim, Stefan Szeider, and Anders Yeo.
\newblock Solving {MAX}-{$r$}-{SAT} above a tight lower bound.
\newblock 61(3):638--655, 2011.

\bibitem{Backstrom92}
Christer B{\"a}ckstr{\"o}m.
\newblock {\em Computational Complexity of Reasoning about Plans}.
\newblock PhD thesis, Lin\-k{\"o}ping University, Lin\-k{\"o}ping, Sweden,
  1992.

\bibitem{Backstrom95}
Christer B{\"a}ckstr{\"o}m.
\newblock Expressive equivalence of planning formalisms.
\newblock {\em Artif. Intell.}, 76(1-2):17--34, 1995.

\bibitem{BackstromChenJonssonOrdyniakSzeider12}
Christer B{\"a}ckstr{\"o}m, Yue Chen, Peter Jonsson, Sebastian Ordyniak, and
  Stefan Szeider.
\newblock The complexity of planning revisited - a parameterized analysis.
\newblock In J{\"o}rg Hoffmann and Bart Selman, editors, {\em Proceedings of
  the 26th AAAI Conference on Artificial Intelligence (AAAI 2012), July 22-26,
  2012, Toronto, ON, Canada}, pages 1735--1741. AAAI Press, 2012.

\bibitem{BackstromJonsson11}
Christer B{\"a}ckstr{\"o}m and Peter Jonsson.
\newblock All {PSPACE}-complete planning problems are equal but some are more
  equal than others.
\newblock In Daniel Borrajo, Maxim Likhachev, and Carlos~Linares L{\'o}pez,
  editors, {\em Proceedings of the 4th Annual Symposium on Combinatorial Search
  (SOCS 2011), Castell de Cardona, Barcelona, Spain, July 15-16, 2011}, pages
  10--17. AAAI Press, 2011.

\bibitem{BackstromJonssonStahlberg13}
Christer B{\"a}ckstr{\"o}m, Peter Jonsson, and Simon St{\aa}hlberg.
\newblock Fast detection of unsolvable planning instances using local
  consistency.
\newblock In {\em Proceedings of the 6th Annual Symposium on Combinatorial
  Search (SOCS 2013), Leavenworth, WA, USA, July 11-13, 2013}, pages 29--37,
  2013.

\bibitem{BackstromKlein91}
Christer B{\"a}ckstr{\"o}m and Inger Klein.
\newblock Planning in polynomial time: the {SAS-PUBS} class.
\newblock {\em Comput. Intell.}, 7:181--197, 1991.

\bibitem{BackstromNebel95}
Christer B{\"a}ckstr{\"o}m and Bernhard Nebel.
\newblock Complexity results for {SAS$^+$} planning.
\newblock {\em Comput. Intell.}, 11:625--656, 1995.

\bibitem{BetzHelmert09}
Christoph Betz and Malte Helmert.
\newblock Planning with {\it h}$^{\mbox{ + }}$ in theory and practice.
\newblock In B{\"a}rbel Mertsching, Marcus Hund, and Muhammad~Zaheer Aziz,
  editors, {\em Proceedings of KI 2009: Advances in Artificial Intelligence,
  32nd Annual German Conference on AI, Paderborn, Germany, September 15-18,
  2009}, volume 5803 of {\em Lecture Notes in Computer Science}, pages 9--16.
  Springer, 2009.

\bibitem{Bodlaender09}
Hans~L. Bodlaender.
\newblock Kernelization: New upper and lower bound techniques.
\newblock In Jianer Chen and Fedor~V. Fomin, editors, {\em Parameterized and
  Exact Computation, 4th International Workshop (IWPEC 2009), Copenhagen,
  Denmark, September 10-11, 2009, Revised Selected Papers}, volume 5917 of {\em
  Lecture Notes in Computer Science}, pages 17--37. Springer, 2009.

\bibitem{BodlaenderDowneyFellowsHermelin09}
Hans~L. Bodlaender, Rodney~G. Downey, Michael~R. Fellows, and Danny Hermelin.
\newblock On problems without polynomial kernels.
\newblock {\em J. Comput. Syst. Sci.}, 75(8):423--434, 2009.

\bibitem{BrafmanDomshlak03}
Ronen~I. Brafman and Carmel Domshlak.
\newblock Structure and complexity in planning with unary operators.
\newblock {\em J. Artif. Intell. Res.}, 18:315--349, 2003.

\bibitem{Bylander94}
Tom Bylander.
\newblock The computational complexity of propositional {STRIPS} planning.
\newblock {\em Artif. Intell.}, 69(1--2):165--204, 1994.

\bibitem{Bylander96}
Tom Bylander.
\newblock A probabilistic analysis of propositional {STRIPS} planning.
\newblock {\em Artif. Intell.}, 81(1-2):241--271, 1996.

\bibitem{CaiEtAl93b}
Liming Cai, Jianer Chen, Rodney~G. Downey, and Michael~R. Fellows.
\newblock Advice classes of parameterized tractability.
\newblock {\em Annals of Pure and Applied Logic}, 84(1):119--138, 1997.

\bibitem{ChenGimenez10}
Hubie Chen and Omer Gim{\'e}nez.
\newblock Causal graphs and structurally restricted planning.
\newblock {\em J. Comput. Syst. Sci.}, 76(7):579--592, 2010.

\bibitem{deHahnRoubickovaSzeider13}
Ronald de~Haan, Anna Roub\'{\i}ckov{\'a}, and Stefan Szeider.
\newblock Parameterized complexity results for plan reuse.
\newblock In Marie desJardins and Michael~L. Littman, editors, {\em Proceedings
  of the 27th AAAI Conference on Artificial Intelligence (AAAI 2013), Bellevue,
  WA, USA, July 14-18, 2013}, 2013.

\bibitem{DowneyFellowsStege99}
R.~Downey, M.~R. Fellows, and U.~Stege.
\newblock Parameterized complexity: A framework for systematically confronting
  computational intractability.
\newblock In {\em Contemporary Trends in Discrete Mathematics: From DIMACS and
  DIMATIA to the Future}, volume~49 of {\em AMS-DIMACS}, pages 49--99. American
  Mathematical Society, 1999.

\bibitem{DowneyFellows99}
R.~G. Downey and M.~R. Fellows.
\newblock {\em Parameterized Complexity}.
\newblock Monographs in Computer Science. Springer, New York, 1999.

\bibitem{DreyfusWagner71}
S.~E. Dreyfus and R.~A. Wagner.
\newblock The {S}teiner problem in graphs.
\newblock {\em Networks}, 1(3):195--207, 1971.

\bibitem{Edelkamp03}
Stefan Edelkamp.
\newblock Taming numbers and durations in the model checking integrated
  planning system.
\newblock {\em J. Artif. Intell. Res.}, 20:195--238, 2003.

\bibitem{EdelkampKellershoffSulewski10}
Stefan Edelkamp, Mark Kellershoff, and Damian Sulewski.
\newblock Program model checking via action planning.
\newblock In Ron van~der Meyden and Jan-Georg Smaus, editors, {\em Model
  Checking and Artificial Intelligence - 6th International Workshop (MoChArt
  2010), Atlanta, GA, USA, July 11, 2010, Revised Selected and Invited Papers},
  volume 6572 of {\em Lecture Notes in Computer Science}, pages 32--51.
  Springer, 2010.

\bibitem{EdelkampLeueVisser07}
Stefan Edelkamp, Stefan Leue, and Willem Visser, editors.
\newblock {\em Directed Model Checking, 26.04. - 29.04.2006}, volume 06172 of
  {\em Dagstuhl Seminar Proceedings}. Internationales Begegnungs- und
  Forschungszentrum f{\"u}r Informatik (IBFI), Schloss Dagstuhl, Germany, 2007.

\bibitem{Fellows06}
Michael~R. Fellows.
\newblock The lost continent of polynomial time: Preprocessing and
  kernelization.
\newblock In {\em Proceedings of Parameterized and Exact Computation, 2nd
  International Workshop (IWPEC 2006), Z{\"u}rich, Switzerland, September
  13-15, 2006}, volume 4169 of {\em Lecture Notes in Computer Science}, pages
  276--277. Springer, 2006.

\bibitem{FlumGrohe06}
J\"{o}rg Flum and Martin Grohe.
\newblock {\em Parameterized Complexity Theory}, volume XIV of {\em Texts in
  Theoretical Computer Science. An EATCS Series}.
\newblock Springer, Berlin, 2006.

\bibitem{Fomin10}
Fedor~V. Fomin.
\newblock Kernelization.
\newblock In Farid~M. Ablayev and Ernst~W. Mayr, editors, {\em Proceedings of
  Computer Science - Theory and Applications, 5th International Computer
  Science Symposium in Russia (CSR 2010), Kazan, Russia, June 16-20, 2010},
  volume 6072 of {\em Lecture Notes in Computer Science}, pages 107--108.
  Springer, 2010.

\bibitem{FortnowSanthanam08}
Lance Fortnow and Rahul Santhanam.
\newblock Infeasibility of instance compression and succinct {PCPs} for {NP}.
\newblock In Cynthia Dwork, editor, {\em Proceedings of the 40th Annual ACM
  Symposium on Theory of Computing, Victoria, BC, Canada, May 17-20, 2008},
  pages 133--142. ACM, 2008.

\bibitem{GhallabNauTraverso04}
Malik Ghallab, Dana~S. Nau, and Paolo Traverso.
\newblock {\em Automated planning - theory and practice}.
\newblock Elsevier, 2004.

\bibitem{GimenezJonsson08}
Omer Gim{\'e}nez and Anders Jonsson.
\newblock The complexity of planning problems with simple causal graphs.
\newblock {\em J. Artif. Intell. Res.}, 31:319--351, 2008.

\bibitem{GuoNiedermeier07}
Jiong Guo and Rolf Niedermeier.
\newblock Invitation to data reduction and problem kernelization.
\newblock {\em ACM SIGACT News}, 38(2):31--45, March 2007.

\bibitem{GuoNiedermeierSuchy11}
Jiong Guo, Rolf Niedermeier, and Ondrej Such{\'y}.
\newblock Parameterized complexity of arc-weighted directed {S}teiner problems.
\newblock {\em SIAM J. Discrete. Math.}, 25(2):583--599, 2011.

\bibitem{GuptaNau92}
Naresh Gupta and Dana~S. Nau.
\newblock On the complexity of blocks-world planning.
\newblock {\em Artif. Intell.}, 56(2-3):223--254, 1992.

\bibitem{Helmert03}
Malte Helmert.
\newblock Complexity results for standard benchmark domains in planning.
\newblock {\em Artif. Intell.}, 143(2):219--262, 2003.

\bibitem{Helmert04}
Malte Helmert.
\newblock A planning heuristic based on causal graph analysis.
\newblock In Shlomo Zilberstein, Jana Koehler, and Sven Koenig, editors, {\em
  Proceedings of the 14th International Conference on Automated Planning and
  Scheduling (ICAPS 2004), Whistler, BC, Canada, June 3-7 2004}, pages
  161--170. AAAI, 2004.

\bibitem{Helmert06b}
Malte Helmert.
\newblock The {F}ast {D}ownward planning system.
\newblock {\em J. Artif. Intell. Res.}, 26:191--246, 2006.

\bibitem{Helmert06}
Malte Helmert.
\newblock New complexity results for classical planning benchmarks.
\newblock In Derek Long, Stephen~F. Smith, Daniel Borrajo, and Lee McCluskey,
  editors, {\em Proceedings of the 16th International Conference on Automated
  Planning and Scheduling (ICAPS 2006), Cumbria, UK, June 6-10, 2006}, pages
  52--62. AAAI, 2006.

\bibitem{Jonsson99}
Peter Jonsson.
\newblock Strong bounds on the approximability of two {PSPACE}-hard problems in
  propositional planning.
\newblock {\em Ann. Math. Artif. Intell.}, 26(1-4):133--147, 1999.

\bibitem{JonssonBackstrom98}
Peter Jonsson and Christer B{\"a}ckstr{\"o}m.
\newblock State-variable planning under structural restrictions: Algorithms and
  complexity.
\newblock {\em Artif. Intell.}, 100(1--2):125--176, 1998.

\bibitem{JonssonBackstrom98b}
Peter Jonsson and Christer B{\"a}ckstr{\"o}m.
\newblock Tractable plan existence does not imply tractable plan generation.
\newblock {\em Ann. Math. Artif. Intell.}, 22(3-4):281--296, 1998.

\bibitem{KroneggerPfandlerPichler13}
Martin Kronegger, Andreas Pfandler, and Reinhard Pichler.
\newblock Parameterized complexity of optimal planning: A detailed map.
\newblock In Francesca Rossi, editor, {\em Proceedings of the 23rd
  International Joint Conference on Artificial Intelligence (IJCAI 2013),
  Beijing, China, August 3-9, 2013}, pages 954--961. IJCAI/AAAI, 2013.

\bibitem{McAllesterRosenblitt91}
David~A. McAllester and David Rosenblitt.
\newblock Systematic nonlinear planning.
\newblock In Thomas~L. Dean and Kathleen McKeown, editors, {\em Proceedings of
  the 9th National Conference on Artificial Intelligence (AAAI 1991), Anaheim,
  CA, USA, July 14-19, 1991, Volume 2}, pages 634--639. AAAI Press / The MIT
  Press, 1991.

\bibitem{Pietrzak03}
Krzysztof Pietrzak.
\newblock On the parameterized complexity of the fixed alphabet shortest common
  supersequence and longest common subsequence problems.
\newblock {\em J. Comput. Syst. Sci.}, 67(4):757--771, 2003.

\bibitem{WehrleHelmert09}
Martin Wehrle and Malte Helmert.
\newblock The causal graph revisited for directed model checking.
\newblock In Jens Palsberg and Zhendong Su, editors, {\em Proceedings of Static
  Analysis, 16th International Symposium (SAS 2009), Los Angeles, CA, USA,
  August 9-11, 2009.}, volume 5673 of {\em Lecture Notes in Computer Science},
  pages 86--101. Springer, 2009.

\bibitem{WilliamsNayak97}
Brian~C. Williams and P.~Pandurang Nayak.
\newblock A reactive planner for a model-based executive.
\newblock In {\em Proceedings of the 15th International Joint Conference on
  Artificial Intelligence (IJCAI 1997), Nagoya, Japan, August 23-29, 1997},
  pages 1178--1185. Morgan Kaufmann, 1997.

\end{thebibliography}

\end{document}